\documentclass{article}

    \usepackage[preprint]{neurips_2025}

\usepackage[utf8]{inputenc} 
\usepackage[T1]{fontenc}    
\usepackage{hyperref}       
\usepackage{url}            
\usepackage{booktabs}       
\usepackage{amsfonts}       
\usepackage{nicefrac}       
\usepackage{microtype}      
\usepackage{xcolor}         

\usepackage{booktabs} 
\usepackage{multirow} 
\usepackage{graphicx} 
\usepackage{adjustbox} 
\usepackage{subcaption}
\usepackage{booktabs}
\usepackage{amsmath, amsthm, amssymb}
\newtheorem{prop}{Proposition}
\usepackage{nccmath}
\PassOptionsToPackage{table}{xcolor} 
\usepackage{algorithm}
\usepackage{algorithmicx}
\usepackage{algpseudocode}
\usepackage{threeparttable}

\title{Effective Probabilistic Time Series Forecasting with Fourier Adaptive Noise-Separated Diffusion}

\author{%
  Xinyan Wang\textsuperscript{1,2}\thanks{Work done when Xinyan Wang was an intern at AMAP, Alibaba Group.},
  Rui Dai\textsuperscript{2}\thanks{Corresponding author.},
  Kaikui Liu\textsuperscript{2},
  Xiangxiang Chu\textsuperscript{2} \\
  \textsuperscript{1}AMSS, Chinese Academy of Sciences, Beijing, China \\
  \textsuperscript{2} AMAP, Alibaba Group, Beijing, China \\
  \textsuperscript{1}\texttt{wangxinyan@amss.ac.cn},
  \textsuperscript{2}\texttt{\{daima.dr,damon, chuxiangxiang.cxx\}@alibaba-inc.com,}
}

\begin{document}

\maketitle

\begin{abstract}
We propose the \textbf{F}ourier \textbf{A}daptive \textbf{L}ite \textbf{D}iffusion \textbf{A}rchitecture (\textbf{FALDA}), a novel probabilistic framework for time series forecasting. First, we introduce the Diffusion Model for Residual Regression (DMRR) framework, which unifies diffusion-based probabilistic regression methods. Within this framework, FALDA leverages Fourier-based decomposition to incorporate a component-specific architecture, enabling tailored modeling of individual temporal components.
A conditional diffusion model is utilized to estimate the future noise term, while our proposed lightweight denoiser, DEMA (Decomposition MLP with AdaLN), conditions on the historical noise term to enhance denoising performance.
Through mathematical analysis and empirical validation, we demonstrate that FALDA effectively reduces epistemic uncertainty, allowing probabilistic learning to primarily focus on aleatoric uncertainty.
Experiments on six real-world benchmarks demonstrate that FALDA consistently outperforms existing probabilistic forecasting approaches across most datasets for long-term time series forecasting while achieving enhanced computational efficiency without compromising accuracy. Notably, FALDA also achieves superior overall performance compared to state-of-the-art (SOTA) point forecasting approaches, with improvements of up to 9\%. 
\end{abstract}

\section{Introduction}\label{sec:intro}
Time series forecasting (TSF) is a fundamental challenge in practical applications, playing a crucial role in decision-making systems across multiple domains, including finance \cite{2020Multivariatefinance}, healthcare \cite{2023Medical}, and transportation \cite{lv2014traffic}\cite{dai2020hybrid}. Recent developments in deep learning have yielded various effective approaches for TSF, with deterministic models such as Autoformer \cite{wu2021autoformer}, DLinear \cite{2023DLinear}, and iTransformer \cite{liu2024itransformer}, demonstrating notable performance. These models process historical time series data to generate future predictions, exhibiting strong capabilities in point forecasting tasks.

Diffusion models have demonstrated significant success across various generative tasks, including image generation~\cite{esser2024scaling,rombach2022high,peebles2023scalable,chu2024visionllama,liuflow,lan2025flux,ramesh2021zero,flux2024,chu2025usp} and video generation~\cite{zhang2024trip,2025motionpro,zheng2024open,bar2024lumiere,hu2024animate,blattmann2023stable,yang2024cogvideox,lin2024open}. Recently, their application has extended to probabilistic forecasting for long-term time series prediction~\cite{2024mgtsd,tashiro2021csdi,2024mrdiff}, where most approaches focus on reconstructing complete temporal patterns encompassing seasonal variations, trend components and noise patterns. While these models exhibit strong performance in probabilistic forecasting, their point forecasting accuracy generally trails that of deterministic estimation methods.
This limitation stems from the inherent incompatibility between the progressive noise injection mechanism in diffusion models and time-series data characteristics. The gradual noise addition process tends to disrupt the temporal structure, making it particularly difficult to recover meaningful patterns from pure noise. This challenge becomes especially pronounced when handling non-stationary time series, as their statistical properties (e.g., mean, variance, autocorrelation, etc.) evolve over time \cite{2024series_to_series, 24diffusionts, 2022NSformer, 2024fan}.

Recent studies have explored hybrid approaches combining point estimation with diffusion models. TMDM \cite{2024TMDM} incorporates predictions from point estimation models into both forward and backward diffusion processes to enhance future forecasting. 
D$^3$U \cite{2025D3U} attempts to decouple deterministic and uncertainty learning by leveraging embedded representations from point estimation to guide the diffusion model in capturing residual patterns, thereby avoiding the need to reconstruct complete temporal components through diffusion. 
Although demonstrating superior point estimation capability compared to previous diffusion-based approaches \cite{tashiro2021csdi, rasul2021timegrad}, these approaches (1) do not specifically address temporal dynamics such as non-stationary patterns, and (2) do not adequately separate epistemic uncertainty from aleatoric uncertainty. 
This limitation prevents the diffusion model from focusing on uncertainty learning, which hinders further improvements in point estimation performance, particularly when applied to strong backbone models.

In this paper, we first analyze the decoupling mechanisms for deterministic and uncertain components in existing approaches \cite{2024TMDM,2025D3U,2020ddpm,2022card}, and introduce a unified generalized diffusion learning framework called DMRR (Diffusion Model for Residual Regression). Building on DMRR, we develop FALDA, a novel diffusion-based time series forecasting framework that employs Fourier decomposition to decouple time series into three distinct components: non-stationary trends, stationary patterns, and noise patterns. Through tailored modeling of each component, FALDA effectively separates epistemic uncertainty and aleatoric uncertainty \cite{gawlikowski2023aleatoric}, allowing the probabilistic modeling component to focus exclusively on aleatoric uncertainty. A lightweight denoiser DEMA is designed to handle multi-scale residuals. 
As a non-autoregressive diffusion model, FALDA avoids the common issue of error accumulation and demonstrates superior performance in long-range prediction tasks. Unlike conventional approaches that predict diffusion noise \cite{2024TMDM, tashiro2021csdi}, our denoiser directly constructs the target series, thereby reducing the learning complexity for temporal patterns \cite{shen2023timediff}. By integrating DDIM \cite{2021ddim} and DEMA, FALDA achieves both training and sampling efficiency. As illustrated in Figure \ref{fig:hexagonal_performance_falda}, FALDA outperforms existing methods in both point estimation and probabilistic forecasting.

In summary, our main contributions are:
\begin{itemize}
        \item We introduce the Diffusion Model for Residual Regression (DMRR) framework to unify recent diffusion approaches for probabilistic regression and provide rigorous mathematical proofs demonstrating the equivalence of their underlying diffusion processes.
        \item We propose the Fourier Adaptive Lite Diffusion Architecture (FALDA), a diffusion-based probabilistic time series forecasting framework that leverages Fourier decomposition to decouple and model different time-series components. We design DEMA (Decomposition MLP with AdaLN), a lightweight denoiser that integrates adaptive layer normalization and trend-seasonality decomposition to handle multi-scale residuals. Combined with DDIM, DEMA improves computational efficiency while maintaining performance.
        \item FALDA supports plug-and-play deployment through a phase-adaptive training schedule, enabling seamless integration (e.g., processing stationary term with SOTA deterministic models). We evaluate our model on six real-world datasets, and the results demonstrate that our model achieves superior overall performance on both accuracy and probabilistic metrics against the state-of-the-art models.
\end{itemize}
\begin{figure}[hbtp]
    \centering
    \includegraphics[width=1.0 \linewidth]{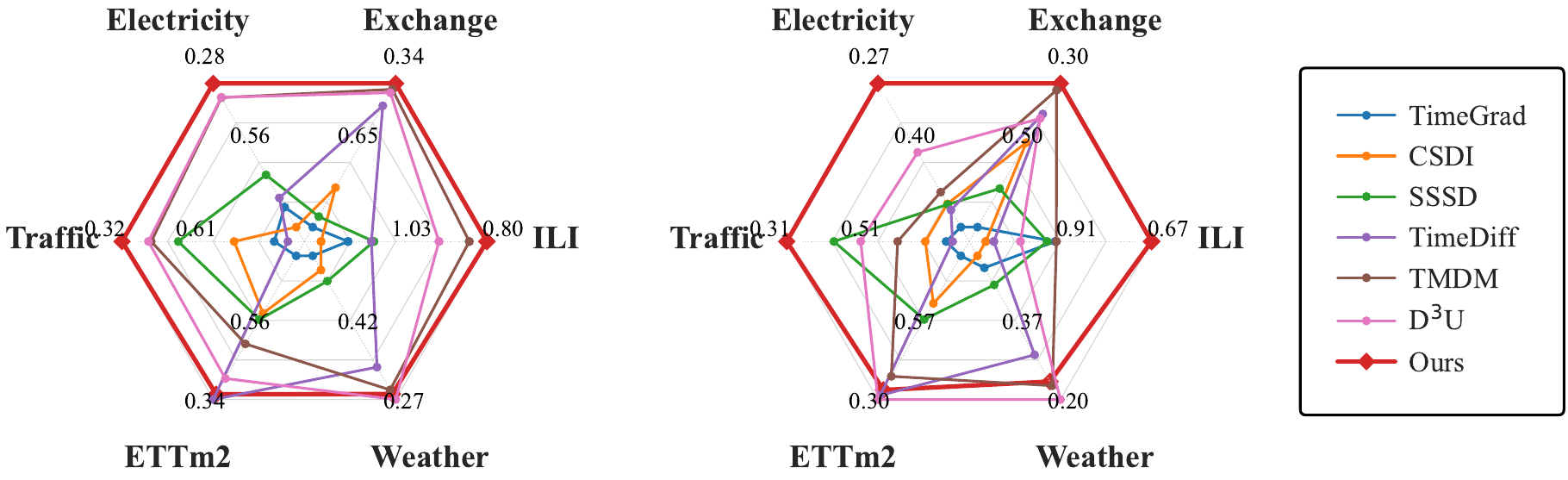}
    \caption{Performance of FALDA in point estimation (MAE, left) and probabilistic prediction (CRPS, right). All three plug-and-play methods (TMDM, D$^3$U, and FALDA) utilize NSformer as the same backbone network for fair comparison.}
    \label{fig:hexagonal_performance_falda}
\end{figure}
\begin{figure}[hbtp]
    \centering
    \includegraphics[width=0.9\linewidth]{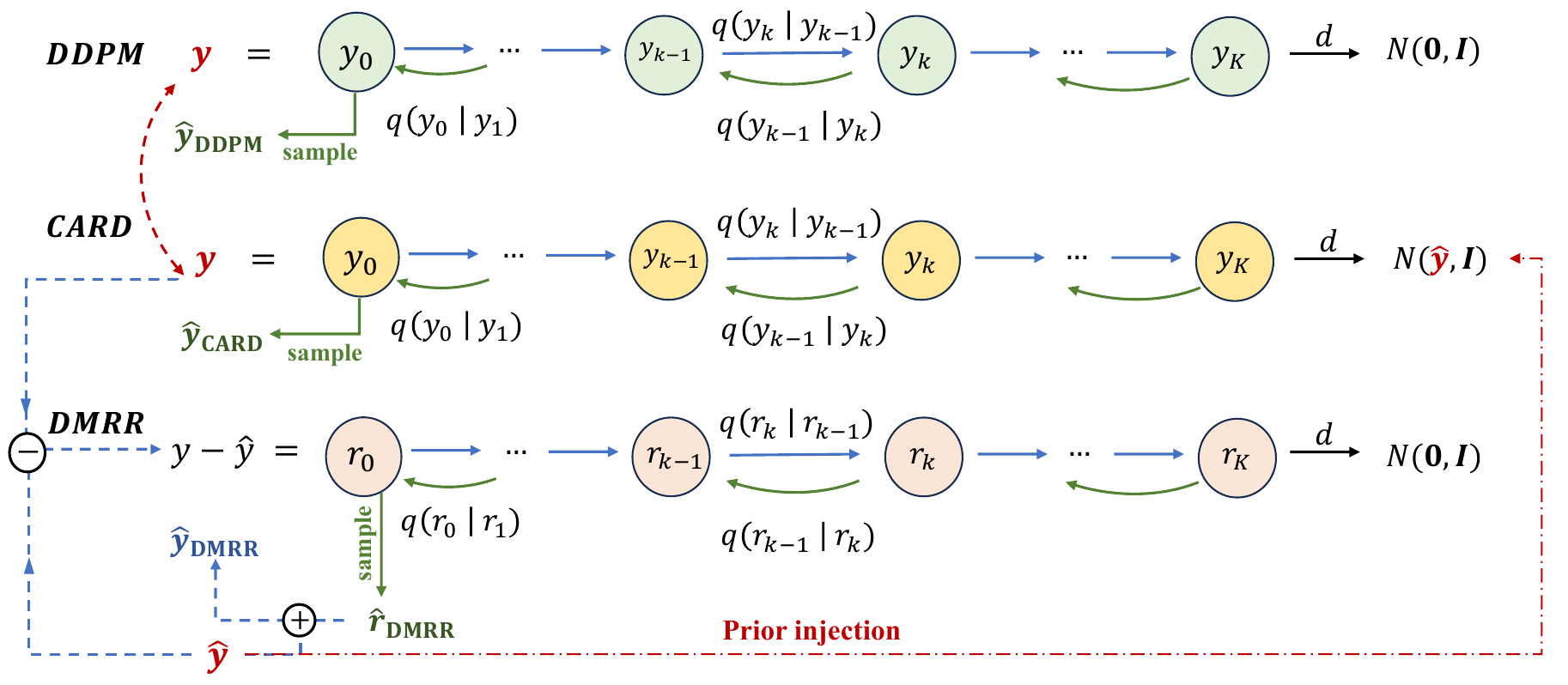}
    \caption{Comparison of three diffusion frameworks: DDPM, CARD, and DMRR, where $\hat{y}_{\text{DDPM}}$, $\hat{y}_{\text{CARD}}$, and $\hat{y}_{\text{DMRR}}$ represent their respective final estimates.
    }
    \label{fig:1-DMRR}
\end{figure}
\section{Diffusion Model for Residual Regression (DMRR)} \label{sec:DMRR}
Diffusion models have gained considerable traction in probabilistic regression tasks, particularly in the domain of probabilistic time series forecasting, where they have emerged as a notably effective paradigm for handling sequential dependencies through their iterative denoising mechanism. While some recent probabilistic regression methods have demonstrated state-of-the-art performances \cite{2022card,2024TMDM,2025D3U}, they inherently conform to a unified framework that refines residual errors through Denoising Diffusion Probabilistic Models (DDPM) \cite{2020ddpm}. In this work, we term this framework ‌Diffusion Model for Residual Regression (DMRR).
This section begins with a formal review of CARD \cite{2022card}, which establishes a generalized framework extending DDPM, where DDPM can be viewed as a special case with zero prior knowledge. Through the lens of the DMRR framework, we subsequently demonstrate that CARD essentially applies standard DDPM to perform residual fitting, establishing a conceptual unification across these seemingly disparate approaches \cite{2024TMDM, 2025D3U}. 

\paragraph{CARD} \label{sec:card}
The Classification and Regression Diffusion (CARD) model \cite{2022card} extends Denoising Diffusion Probabilistic Models (DDPM) by incorporating prior knowledge into both forward and reverse diffusion processes (see Appendix \ref{appendix:ddpm} for DDPM fundamentals). Formally, given a target variable $y_0 \sim q(y)$ with covariate $x$, CARD utilizes prior knowledge $f_\phi(x)$ to guide the generation, where $f_{\phi}$ can be a pretrained network as demonstrated in \cite{2022card}. This yields the following forward diffusion process:
\begin{equation}
    \label{eq:card_forward}
    \begin{aligned}
        y_k &= \sqrt{\alpha_k} y_{k-1} + (1-\sqrt{1-\beta_k})f_{\phi}(x) + \sqrt{\beta_k} z_k, \quad z_k \sim \mathcal{N}(0, 1), \quad \text{(one-step)} \\
        y_k &= \sqrt{\bar{\alpha}_k}y_0 + (1-\sqrt{\bar{\alpha}_k})f_{\phi}(x) + \sqrt{1-\bar{\alpha}}\bar{z}_k, \quad \bar{z}_k \sim \mathcal{N}(0, 1), \quad \text{(multi-step)}
    \end{aligned}
\end{equation}

where $\alpha_k = 1-\beta_k \in (0,1)$ and $\bar{\alpha}_k = \prod_{s=1}^k \alpha_s$ denote the noise schedule parameters for $k=1, 2, \dots, K$. This process converges to a Gaussian limit distribution: $\mathcal{N}(f_{\phi}(x), I).$ The corresponding reverse process posterior distribution is given by:
\begin{equation} \label{eq:card-posterior}
    \begin{gathered}
        q(y_{k-1}| y_k, y_0) = \mathcal{N}(y_{k-1}; \tilde{m}_k, \tilde{\beta}_k I), \text{where} \\
        \tilde{m}_k = \frac{\beta_k \sqrt{\bar{\alpha}_{k-1}}}{1-\bar{\alpha}_k}y_0 + \frac{(1-\bar{\alpha}_{k-1})\sqrt{\alpha_k}}{1-\bar{\alpha}_k}y_k + (1 + \frac{(\sqrt{\bar{\alpha}_k} -1)(\sqrt{\alpha_k}+ \sqrt{\bar{\alpha}_{k-1}})}{1-\bar{\alpha}_k})f_{\phi}(x), \\
        \tilde{\beta}_k = \frac{1-\bar{\alpha}_{k-1}}{1-\bar{\alpha}_k}\beta_k.
    \end{gathered}
\end{equation}

The residual $l_k = y_k - f_{\phi}(x)$ exhibits the same convergence behavior as DDPM, with a standard Gaussian distribution as its limit distribution. This equivalence underpins our DMRR framework, which systematically formalizes this residual learning paradigm within a unified diffusion framework. 

\paragraph{The unified framework}
As illustrated in Figure \ref{fig:1-DMRR}, our proposed DMRR framework introduces a residual learning paradigm that decouples prior knowledge from the limit distribution in CARD diffusion process. Given the target $y$, the framework first generates a preliminary estimate $\hat{y}$ (for CARD $\hat{y}=f_{\phi}(x)$). Unlike CARD, which learns the full data distribution $y$ guided by $\hat{y}$, DMRR focuses on learning the residual distribution $q(r)$, where $r = y - \hat{y}$. This is implemented through a DDPM process, where the forward diffusion follows the Markov chain $\{r_0=r, r_1, \dots, r_k, \dots\}$ with $r_k$ denoting the noise sample at step $k$. The reverse process generates residual predictions: $\hat{r}_{\text{DMRR}}$ via the denoising network. The final output, which can be considered as a refinement of the preliminary estimate $\hat{y}$, combines both components:
\begin{equation}
    \hat{y}_{\text{DMRR}} = \hat{y} + \hat{r}_{\text{DMRR}}.
\end{equation}
Mathematically, we prove that $l_k = y_k - \hat{y}$ in CARD and $r_t$ in DMRR possess identical conditional and posterior distributions (see Appendix \ref{appendix:mathematical derivations} for rigorous proofs). And it should be noted that when the preliminary estimate $\hat{y}=0$, CARD and DMRR degenerate to the standard DDPM. 

In Section \ref{sec:difference with existing work}, we comprehensively discuss the diffusion framework underlying state-of-the-art time series forecasting models. We further analyze how different framework designs affect the performance of time series prediction tasks and illustrate the advantages of DMRR framework in time series forecasting tasks.
\begin{figure}[hbtp]
    \centering
    \includegraphics[width=1.0\linewidth]{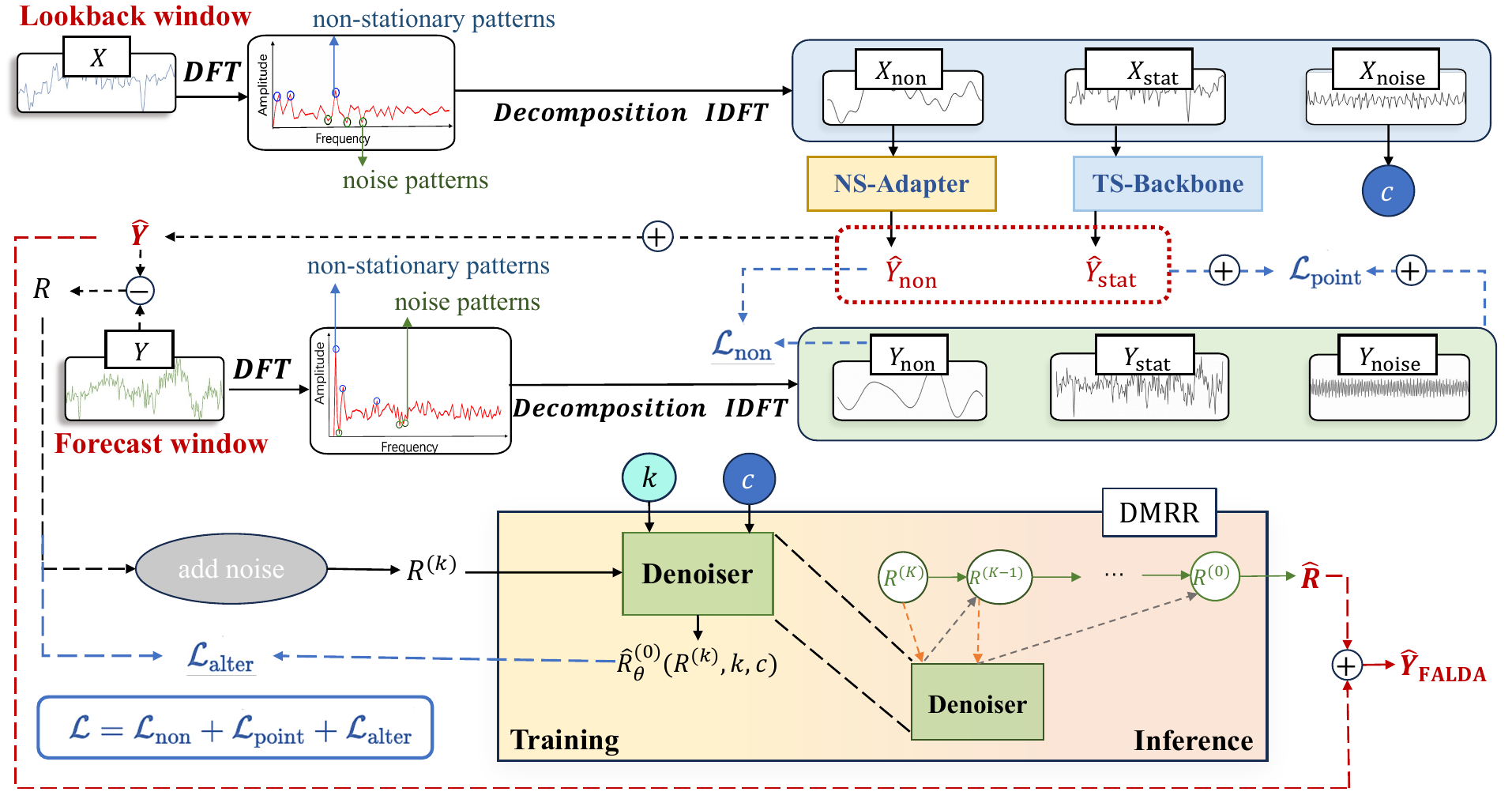}
    \caption{An illustration of the proposed FALDA framework. By leveraging Fourier decomposition, NS-Adapter and TS-Backbone generate the preliminary estimation, $\hat{Y}$. The prediction residual $R=Y-\hat{Y}$ is then input into the denoiser for subsequent probabilistic learning and refinement of the preliminary estimation.}
    \label{fig:2-FALDA}
\end{figure}

\section{Fourier Adaptive Lite Diffusion Architecture (FALDA)} \label{sec:falda}
From a methodological perspective, probabilistic time series forecasting is a specialized form of probabilistic regression applied to temporal data, necessitating explicit modeling of sequential dependencies.
Within the DMRR framework, we propose the Fourier Adaptive Lite Diffusion Architecture (FALDA), which leverages point-guided diffusion models for TSF while reducing the influence of non-stationarity and noise on probabilistic learning. We further analyze the underlying mechanism through a comparative discussion of diffusion-based TSF models.

\subsection{Problem Statement} \label{sec:falda:problem-statement}
In the time series forecasting task, let $X = \{X^{t}\}_{t=1}^{T} \in \mathbb{R}^{T \times D}$ represent an observed multivariate time series with $T$ historical time steps, where each $X^t \in \mathbb{R}^D$ denotes the $D$-dimensional observation vector at time $t$. Given this lookback window $X$, the objective is to forecast the subsequent $S$ time steps, denoted as $Y = \{Y^{t}\}_{t=1}^{S} \in \mathbb{R}^{S \times D}$. 

\subsection{Main framework}\label{sec:FALDA-main-framework}
\paragraph{FALDA}
As shown in Figure~\ref{fig:2-FALDA}, the time series is first decomposed into three components: a \textbf{non-stationary term}, $Y_{\text{non}}$, representing temporal components that exhibit time-varying statistical properties; a \textbf{stationary term}, $Y_{\text{stat}}$, comprising components whose statistical properties remain invariant over time; and a \textbf{noise term}, $Y_{\text{noise}}$, reflecting inherent stochastic disturbances within the time series. 
Following \cite{24diffusionts, 2024fan}, this decomposition is performed using the Fourier transform. Specifically, the non-stationary component is extracted by reconstructing the time series from the frequencies corresponding to the $K_1$ largest amplitudes \cite{2024fan}, while the noise component is obtained by reconstructing the time series from the frequencies associated with the $K_2$ smallest amplitudes \cite{24diffusionts, 2025D3U}:
\begin{equation} \label{eq:decomposition}
    Y_{\text{non}} = \mathcal{F}^{-1}(\mathrm{Top}(\mathcal{F}(Y), K_1)), \quad Y_{\text{noise}} = \mathcal{F}^{-1}(\mathrm{Bottom}(\mathcal{F}(Y), K_2)).
\end{equation}
Here, $\mathcal{F}$ denotes the Fourier transform and $\mathcal{F}^{-1}$ denotes the inverse Fourier transform. The operators $\mathrm{Top}(\cdot, K_1)$ and $\mathrm{Bottom}(\cdot, K_2)$ select the frequency components with the $K_1$ largest and the $K_2$ smallest amplitudes, respectively. 
The stationary term is defined as:
\begin{equation} \label{eq:y_norm}
Y_{\text{stat}} = Y - Y_{\text{non}} - Y_{\text{noise}}.
\end{equation}
Similarly, the decomposition for $X$ is given by $X = X_{\text{non}} + X_{\text{stat}} + X_{\text{noise}}$. 
Based on this decomposition, FALDA integrates three key components: (1) a non-stationary adapter (\textbf{NS-Adapter}) $f_{w}$, which models the non-stationary term $Y_{\text{non}}$ by addressing evolving temporal patterns and mitigating epistemic uncertainty; (2) a time series backbone (\textbf{TS-Backbone}) $g_{\phi}$, which captures temporally invariant patterns to model the stationary component $Y_{\text{stat}}$; (3) a conditional diffusion process with a lightweight denoiser $\hat{R}^{(0)}_{\theta}$, which specializes in handling aleatoric uncertainty by modeling the inherent noise component $Y_{\text{noise}}$ in the data. The predictions for the non-stationary and stationary components are given by:
\begin{equation}\label{eq:non-norm-adapter}
    \hat{Y}_{\text{non}} = f_{w}(X_{\text{non}}), \quad \hat{Y}_{\text{stat}} = g_{\phi}(X_{\text{stat}}).
\end{equation}
Here $f_w$ is implemented as a multi-layer perceptron (MLP) to effectively capture non-stationary patterns, while $g_{\phi}$ serves as a flexible backbone that can be substituted by conventional point forecasting models. For further details on the implementation of $f_w$, please refer to Appendix~\ref{appendix:model-non-stationary-adapter}.

Eq. \eqref{eq:non-norm-adapter} gives a preliminary estimation $\hat{Y}=\hat{Y}_{\text{non}} + \hat{Y}_{\text{stat}}$. We use DDPM to model the residual component, which is defined as $R = Y - \hat{Y}$. During the reverse process, the posterior mean is parameterized as: $\tilde{\mu}_{\theta}(R^{(k)}, k) = \frac{\sqrt{\bar{\alpha}_{k-1}}\beta_k}{1-\bar{\alpha}_k} \hat{R}^{(0)}_{\theta} (R^{(k)}, k, c) + \frac{\sqrt{\alpha_k}(1-\bar{\alpha}_{k-1})}{1-\bar{\alpha}_k}R^{(k)}, k =K, K-1, ..., 1$. $R^{(k)}$ represents the noise sample at step $k$, and condition $c$ is set to the noise term of the lookback window, $X_{\text{noise}}$. The denoiser $\hat{R}^{(0)}_{\theta} (R^{(k)}, k, c)$ directly reconstructs the target $R = R^{(0)}$ instead of learning the diffusion noise at each step. This approach alleviates the learning difficulty of time series data \cite{shen2023timediff, 24diffusionts}. An estimate of the residuals is generated through reverse sampling: $\hat{R}^{(K)} \rightarrow \hat{R}^{(K-1)} \rightarrow \dots \rightarrow \hat{R}^{(0)} = \hat{R}$. The final output is the sum of the three component outputs in FALDA:
\begin{equation}
    \hat{Y}_{\text{FALDA}} = \hat{Y}_{\text{non}} + \hat{Y}_{\text{stat}} + \hat{R} .
\end{equation}
In alignment with the multi-component decomposition framework of FALDA, we propose a tailored loss function designed to facilitate multi-task optimization. To effectively capture non-stationary patterns, we define the non-stationary term loss $\mathcal{L}_{\text{non}}$ to provide prior guidance. Simultaneously, to ensure the overall accuracy of the preliminary point estimations, we define the overall point estimation loss $\mathcal{L}_{\text{point}}$. These two loss functions can be expressed as:
\begin{equation} \label{eq:non_norm_loss}
    \mathcal{L}_{\text{non}} = \ell(Y_{\text{non}}, \hat{Y}_{\text{non}}), \quad  \mathcal{L}_{\text{point}} = \ell(Y, \hat{Y}), 
\end{equation}
where $\ell$ is the $L_1$ loss. The alternative loss
$\mathcal{L}_{\text{alter}}$ simultaneously optimizes the denoiser and fine-tunes the point estimate model through two terms:
\begin{equation} \label{eq:loss_diffusion}
    \mathcal{L}_{\text{alter}} = \lambda_s  \underbrace{\|\text{sg}(R) - \hat{R}^{(0)}_{\theta}(R^{(k)}, k, c)\|^2}_{\mathcal{L}_{\text{diffusion}}} + \eta_s \underbrace{ \|R - \text{sg}(\hat{R}^{(0)}_{\theta}(R^{(k')}, k', c))\|^2}_{\mathcal{L}_{\text{finetune}}}.
\end{equation}
Here, $R=Y-\hat{Y}$. The first term $\mathcal{L}_{\text{diffusion}}$ targets the optimization of the denoiser, where the stop-gradient operation $\text{sg}(\cdot)$ ensures no interference with the point estimate model’s training.
The second term $\mathcal{L}_{\text{finetune}}$ fine-tunes the point estimate models, improving them alongside the denoiser. Here, $k'$ is a hyperparameter that enables flexible selection of the diffusion step during the fine-tuning process. Additionally, two scheduling hyperparameters, $\lambda_s$ and $\eta_s$, are introduced to control the alternating optimization of the two losses in $\mathcal{L}_{\text{alter}}$. These parameters depend on the current training epoch $s$, and are governed by a threshold $\delta$ and a period $\Delta$: 
\begin{equation} \label{eq:lambda_s_eta_s}
\lambda_s = \begin{cases} 
1, & s \ge \delta \ \text{and} \ s \bmod \Delta \neq 0 \\ 
0, & \text{otherwise}
\end{cases}, \
\eta_s = \begin{cases} 
1, & s \ge \delta \ \text{and} \ s \bmod \Delta = 0 \\ 
0, & \text{otherwise}
\end{cases},
\end{equation}
where the hyperparameter $\delta$ determines the pretraining duration (in epochs) for the point forecasting models, while $\Delta$ controls the alternating intervals between denoiser training and fine-tuning.
The final loss function is given by:
\begin{equation} \label{eq:loss}
    \mathcal{L} = \mathcal{L}_{\text{non}} + \mathcal{L}_{\text{point}} + \mathcal{L}_{\text{alter}}.
\end{equation}
For complete training and inference algorithm of FALDA, please refer to Appendix~\ref{appendix:alg}.

\paragraph{DEMA}
We design DEMA (Decomposition MLP with AdaLN), a lightweight denoiser denoted as $\hat{R}^{(0)}_{\theta}$,
to effectively predict the future time series noise term $Y_{\text{noise}}$. As a conditional denoiser, $\hat{R}^{(0)}_{\theta}(\cdot)$ takes the $k$-step noise sample $R^{(k)} \in \mathbb{R}^{S \times D}$, the diffusion step $k$, and condition  $c=X_{\text{noise}} \in \mathbb{R}^{T \times D}$ as input. The input $R^{(k)}$ and condition $c$ are projected into a latent space with dimension $H_d$ through the following embedding process:
\begin{equation}
    h_{k}^{[0]} = \text{Linear}( R^{(k)}) \in \mathbb{R}^{H_d \times D},\quad e_k = \text{Linear}(\text{PE}(k)) + \text{Linear}(c)\in \mathbb{R}^{H_d \times D},
\end{equation}
where $\text{PE}(\cdot)$ is sinusoidal embedding \cite{2017transformer, li2024mar}. The embedding $h_k$ and $e_k$ are then processed by an $L$-layer encoder. 
At each layer $l \in \{0, 1, ..., L-1\}$, the encoder performs the following computations:
\begin{equation} \label{eq:denoiser_ma}
    \left[{\tau}_{\text{season}}^{[l]}, {\tau}_{\text{trend}}^{[l]}\right] = 
    \left[{h}_k^{[l]} - \text{MA}_a({h}_k^{[l]}), \text{MA}_a({h}_k^{[l]})\right] ,
\end{equation}
\begin{equation}
    \left[\gamma^{[l]}_i, \beta^{[l]}_i, o^{[l]}_i\right] = \text{Linear}(\text{SiLU}(e_k)),
\end{equation}
\begin{equation} \label{eq:denoiser:layernorm}
    \bar{\tau}_{i}^{[l]} = (\gamma^{[l]}_{i} + 1) \odot \text{LayerNorm}({\tau}_{i}^{[l]}) + \beta^{[l]}_{i}, 
\end{equation}
where $\gamma^{[l]}_i$, $\beta^{[l]}_i$ and $o^{[l]}_i$ represent the scale factor, shift factor, and gating factor, respectively, with $i\in\{\text{season}, \text{trend}\}$.  
$\text{MA}_a$ denotes the moving average operation with kernel size $a$. The output of an encoder layer is computed as:
\begin{equation}
    {h}_k^{[l+1]} = {h}_k^{[l]} + (o^{[l]}_{\text{season}} + o^{[l]}_{\text{trend}}) \odot  \text{Linear}(\bar{\tau}_{\text{season}}^{[l]} + \bar{\tau}_{\text{trend}}^{[l]}) .
\end{equation}
After processing through an adaptive layer normalization decoder, the denoiser generates its final output $\hat{R}^{(0)}_{\theta}(R^{(k)}, k, c) \in \mathbb{R}^{S \times D}$, where $\theta$ represents all trainable parameters in the network.

\subsection{Analysis of Different Diffusion-based Time Series Models with Residual Learning}\label{sec:difference with existing work}
TMDM and D$^3$U are representative diffusion-based time series forecasting models that incorporate residual learning. Specifically, TMDM employs CARD as its underlying diffusion mechanism, while D$^3$U and FALDA utilize DMRR (see Appendix \ref{appendix:tmdm-d3u} for detailed mathematical formulations). As discussed in Section \ref{sec:DMRR}, DMRR and CARD share identical transition probabilities and posterior distributions, indicating that their stochastic dynamics are mathematically equivalent.
Although theoretically equivalent, DMRR's architecture provides crucial modeling advantages and is inherently more suitable for time series forecasting tasks compared to CARD. Real-world time series typically consist of multiple components (trend, seasonality, and inherent noise) that are often corrupted during the diffusion process due to gradual noise addition. This corruption makes it challenging to recover the time series distribution from the noise data \cite{24diffusionts}. Although the preliminary estimate partially captures temporal patterns, it remains difficult for CARD framework to learn the residual distribution from the noisy full time series $Y^{(k)}$, which also represents a limitation of TMDM.

In contrast, D$^3$U and FALDA, which are based on DMRR, alleviate this limitation through their residual learning paradigm. This paradigm explicitly decouples the preliminary estimation from the limiting distribution in CARD and focuses exclusively on modeling the residual between the preliminary estimate and the ground truth. The residual components encompass both epistemic and aleatoric uncertainties \cite{gawlikowski2023aleatoric}. While D$^3$U demonstrates promising performance by utilizing latent representations from the encoder as the condition in the reverse process, its generalized modeling approach primarily captures epistemic uncertainty due to the lack of explicit consideration for distinct temporal components. 
This architectural characteristic limits its ability to explicitly model the pure underlying probability distribution, especially the aleatoric uncertainty component. Furthermore, this limitation may result in diminishing returns when applied to backbone models that already exhibit strong predictive capabilities. An elaborate analysis of this phenomenon is provided in Appendix~\ref{appendix:probability-view}. Our framework extends this approach by introducing dedicated network architectures designed to capture three key temporal components. This enhanced modeling capability enables more balanced learning of both epistemic and aleatoric uncertainties, thereby contributing to improved point estimation accuracy.

\section{Experiments}
\subsection{Experiment Setup}
In this experiment, we evaluate the performance of multivariate time series forecasting using six widely recognized real-world datasets: ILI, Exchange-Rate, ETTm2, Electricity, Traffic, and Weather. More details are provided in Appendix~\ref{appendix:dataset}. We include 13 state-of-the-art TSF models in our baselines including both point forecasting and probabilistic forecasting methods: Informer \cite{zhou2021informer}, Autoformer \cite{wu2021autoformer}, FEDformer \cite{2022FEDformer}, DLinear \cite{2023DLinear}, TimesNet \cite{wu2023timesnet}, PatchTST \cite{2023patchtst},  iTransformer \cite{liu2024itransformer}, TimeGrad \cite{rasul2021timegrad}, CSDI \cite{tashiro2021csdi}, SSSD \cite{sssd}, TimeDiff \cite{shen2023timediff}, TMDM \cite{2024TMDM}, D$^3$U \cite{2025D3U}.

We set the lookback window $T = 96$ and prediction length $S = 192$, except for ILI where $T = S = 36$. Following \cite{2020ddpm}, we use $K = 1000$ diffusion timesteps with a linear noise schedule. FALDA employs iTransformer as its default backbone if not stated otherwise, with DDIM \cite{2021ddim} for inference acceleration. Implementation details are fully provided in Appendix \ref{appendix:implementation}.

\subsection{Main Result} \label{sec:main_result}
\begin{table*}[htbp]
\vspace{-5pt}
\setlength{\tabcolsep}{3pt} 
\renewcommand{\arraystretch}{1.1} 
\caption{Comparison of MAE and MSE across six real-world datasets. \textbf{Bold} denotes the best-performing method for each metric-dataset combination, while \underline{underlined} indicates the second-best.}
\vspace{1pt}
\begin{tiny} 
\begin{center}
\resizebox{0.83 \linewidth}{!}{ 
\begin{tabular}{c|cc|cc|cc|cc|cc|cc} 
\toprule[1.2pt]
\multirow{2}{*}{\scalebox{0.9}{Methods}} & \multicolumn{2}{c}{\scalebox{0.9}{ILI}} & \multicolumn{2}{c}{\scalebox{0.9}{Exchange}} & \multicolumn{2}{c}{\scalebox{0.9}{Electricity}} & \multicolumn{2}{c}{\scalebox{0.9}{Traffic}} & \multicolumn{2}{c}{\scalebox{0.9}{ETTm2}} & \multicolumn{2}{c}{\scalebox{0.9}{Weather}} \\ 
\cmidrule(lr){2-3}\cmidrule(lr){4-5}\cmidrule(lr){6-7}\cmidrule(lr){8-9}\cmidrule(lr){10-11}\cmidrule(lr){12-13}
\scalebox{0.9}{Metric} & \scalebox{0.9}{MSE} & \scalebox{0.9}{MAE} & \scalebox{0.9}{MSE} & \scalebox{0.9}{MAE} & \scalebox{0.9}{MSE} & \scalebox{0.9}{MAE} & \scalebox{0.9}{MSE} & \scalebox{0.9}{MAE} & \scalebox{0.9}{MSE} & \scalebox{0.9}{MAE} & \scalebox{0.9}{MSE} & \scalebox{0.9}{MAE} \\ 
\toprule[1.2pt]
\scalebox{0.9}{Informer} & 4.620 & 1.456 & 1.092 & 0.853 & 0.319 & 0.399 & 0.696 & 0.379 & 0.494 & 0.525 & 0.598 & 0.544 \\ \midrule
\scalebox{0.9}{Autoformer} & 3.366 & 1.210 & 0.537 & 0.526 & 0.227 & 0.332 & 0.616 & 0.382 & 0.269 & 0.327 & 0.276 & 0.336 \\ \midrule
\scalebox{0.9}{FEDformer} & 2.679 & 1.163 & 0.276 & 0.384 & 0.198 & 0.312 & 0.606 & 0.377 & 0.269 & 0.325 & 0.276 & 0.336 \\ \midrule
\scalebox{0.9}{DLinear} & 2.235 & 1.059 & \underline{0.167} & \underline{0.301} & 0.196 & 0.285 & 0.598 & 0.370 & 0.284 & 0.362 & 0.218 & 0.278 \\ \midrule
\scalebox{0.9}{TimesNet} & 2.671 & 0.986 & 0.224 & 0.343 & 0.184 & 0.289 & 0.617 & 0.336 & 0.249 & 0.309 & 0.219 & 0.261 \\ \midrule
\scalebox{0.9}{PatchTST} & 2.374 & 0.918 & 0.181 & 0.303 & 0.205 & 0.307 & 0.463 & 0.311 & 0.251 & 0.312 & 0.223 & 0.258 \\ \midrule
\scalebox{0.9}{iTransformer} & \underline{1.833} & \underline{0.828} & 0.193 & 0.315 & \underline{0.164} & \underline{0.248} & \underline{0.413} & \underline{0.251} & \underline{0.246} & \textbf{0.300} & \underline{0.217} & \textbf{0.247} \\ \midrule
\scalebox{0.9}{TimeGrad} & 2.644 & 1.142 & 2.429 & 0.902 & 0.645 & 0.723 & 0.932 & 0.807 & 1.385 & 0.732 & 0.885 & 0.551 \\ \midrule
\scalebox{0.9}{CSDI} & 2.538 & 1.208 & 1.662 & 0.748 & 0.553 & 0.795 & 0.921 & 0.678 & 1.291 & 0.576 & 0.842 & 0.523 \\ \midrule
\scalebox{0.9}{SSSD} & 2.521 & 1.079 & 0.897 & 0.861 & 0.481 & 0.607 & 0.794 & 0.498 & 0.973 & 0.559 & 0.693 & 0.501 \\ \midrule
\scalebox{0.9}{TimeDiff} & 2.458 & 1.085 & 0.475 & 0.429 & 0.730 & 0.690 & 1.465 & 0.851 & 0.284 & 0.342 & 0.277 & 0.331 \\ \midrule
\scalebox{0.9}{TMDM} & 1.985 & 0.846 & 0.260 & 0.365 & 0.222 & 0.329 & 0.721 & 0.411 & 0.524 & 0.493 & 0.244 & 0.286 \\ \midrule
\scalebox{0.9}{D$^3$U} & 2.103 & 0.935 & 0.254 & 0.358 & 0.179 & 0.267 & 0.468 & 0.299 & \textbf{0.241} & 0.302 & 0.222 & 0.264 \\ \midrule
\scalebox{0.9}{Ours} & \textbf{1.666} & \textbf{0.821} & \textbf{0.165} & \textbf{0.296} & \textbf{0.163} & \textbf{0.248} & \textbf{0.412} & \textbf{0.251} & \underline{0.246} & \underline{0.301} & \textbf{0.215} & \underline{0.255} \\ \bottomrule[1.2pt]
\end{tabular}}
\end{center}
\end{tiny}
\label{tab:mae-mse}
\vspace{-12pt}
\end{table*} 
\paragraph{Forecasting performance and computational efficiency}
We conduct a comprehensive evaluation of the proposed model against state-of-the-art baselines for four metrics: CRPS, CRPS$_{\text{sum}}$, MAE, and MSE. CRPS and CRPS$_{\text{sum}}$ assess the probabilistic forecasting performance, while MAE and MSE evaluate the point forecasting accuracy. See Appendix \ref{appendix:metrics} for detailed metric descriptions. Table~\ref{tab:mae-mse} summarizes MAE and MSE results across six real-world datasets. Our method outperforms all baselines in four out of six datasets (ILI, Exchange, Electricity, and Traffic) for both MAE and MSE. On the remaining two datasets, our method consistently ranks among the top two performers. The most significant improvement is observed on the ILI dataset, where our model achieves a notable 9\% reduction in MSE compared to iTransformer, the second-best model, demonstrating FALDA's powerful ability in point forecasting.
FALDA also presents superior or comparable probabilistic forecasting performance compared to previous diffusion-based models. Table \ref{tab:crps-result} shows the CRPS and CRPS$_{\text{sum}}$ metrics across 6 datasets. 
On Exchange, FALDA promotes an average of 9\% on CRPS and 39\% on CRPS$_{\text{sum}}$. 
In terms of efficiency, FALDA achieves an inference speed-up of up to 26.3$\times$ and a training speed-up of up to 13.7$\times$ compared to TMDM, as detailed in Appendix~\ref{appendix:train_infer_efficiency}.

\begin{table*}[htbp]
\vspace{-2pt}
\setlength{\tabcolsep}{3pt} 
\renewcommand{\arraystretch}{1.1} 
\caption{Comparison of CRPS and CRPS$_\text{sum}$ across six real-world datasets. \textbf{Bold} denotes the best-performing method for each metric-dataset combination, while \underline{underlined} indicates the second-best.}
\vspace{0pt}
\begin{tiny} 
\resizebox{0.98\linewidth}{!}{ 
\begin{tabular}{c|cc|cc|cc|cc|cc|cc} 
\toprule[1.2pt]
\multirow{2}{*}{\scalebox{0.9}{Methods}} & \multicolumn{2}{c}{\scalebox{0.9}{ILI}} & \multicolumn{2}{c}{\scalebox{0.9}{Exchange}} & \multicolumn{2}{c}{\scalebox{0.9}{ETTm2}} & \multicolumn{2}{c}{\scalebox{0.9}{Weather}} & \multicolumn{2}{c}{\scalebox{0.9}{Electricity}} & \multicolumn{2}{c}{\scalebox{0.9}{Traffic}} \\ 
\cmidrule(lr){2-3}\cmidrule(lr){4-5}\cmidrule(lr){6-7}\cmidrule(lr){8-9}\cmidrule(lr){10-11}\cmidrule(lr){12-13}
\scalebox{0.9}{Metric} & \scalebox{0.9}{CRPS} & \scalebox{0.9}{CRPS$_\text{sum}$} & \scalebox{0.9}{CRPS} & \scalebox{0.9}{CRPS$_\text{sum}$} & \scalebox{0.9}{CRPS} & \scalebox{0.9}{CRPS$_\text{sum}$} & \scalebox{0.9}{CRPS} & \scalebox{0.9}{CRPS$_\text{sum}$} & \scalebox{0.9}{CRPS} & \scalebox{0.9}{CRPS$_\text{sum}$} & \scalebox{0.9}{CRPS} & \scalebox{0.9}{CRPS$_\text{sum}$} \\ 
\toprule[1.2pt]
\scalebox{0.9}{TimeGrad} & 0.924 & 0.527 & 0.661 & 0.437 & 0.785 & 1.051 & 0.482 & 0.503 & 0.503 & 1.452 & 0.657 & 1.683 \\ \midrule
\scalebox{0.9}{CSDI} & 1.104 & 0.607 & 0.448 & 0.469 & 0.625 & 0.782 & 0.508 & 0.465 & 0.465 & 0.823 & 0.612 & 1.275 \\ \midrule
\scalebox{0.9}{SSSD} & 0.945 & 0.548 & 0.564 & 0.370 & 0.571 & 0.275 & 0.445 & 0.442 & 0.466 & 0.580 & 0.414 & 0.949 \\ \midrule
\scalebox{0.9}{TimeDiff} & 1.083 & 0.610 & 0.376 & 0.275 & 0.316 & 0.180 & 0.293 & 0.400 & 0.475 & 0.594 & 0.671 & 0.823 \\ \midrule
\scalebox{0.9}{TMDM} & \underline{0.921} & \underline{0.524} & \underline{0.316} & \underline{0.209}
& 0.380 & 0.226 & 0.226 & \underline{0.292} & 0.446 & \textbf{0.137} & 0.552 & \underline{0.179} \\ \midrule
\scalebox{0.9}{D$^3$U} & 0.951 & 0.566 & 0.318 & 0.210 & \textbf{0.243} & \underline{0.141} & \underline{0.207} & \textbf{0.283} & \textbf{0.202} & \underline{0.160} & \textbf{0.232} & 0.186 \\ \midrule
\scalebox{0.9}{Ours} & \textbf{0.721} & \textbf{0.387} & \textbf{0.289} & \textbf{0.126} & \underline{0.244} & \textbf{0.141} & \textbf{0.207} & 0.298 & \underline{0.231} & \underline{0.160} & \underline{0.245} & \textbf{0.163} \\ \bottomrule[1.2pt]
\end{tabular}}
\end{tiny}
\label{tab:crps-result}
\vspace{-11pt}
\end{table*}

\paragraph{Plug-and-play performance}
To evaluate the generality of our framework, we integrate four well-known point forecasting models into the FALDA framework: Autoformer \cite{wu2021autoformer}, Informer \cite{zhou2021informer}, Transformer \cite{2017transformer}, and iTransformer \cite{liu2024itransformer}. Table~\ref{tab:plug-and-play} shows their performance improvements with FALDA. The experimental results demonstrate consistent improvements in both MSE and MAE metrics across the majority of evaluated datasets. 
The most significant improvements are observed for Informer, which achieves maximum reductions of 66.4\% in MSE and 46.2\% in MAE on the same dataset. 
For iTransformer, which serves as a strong baseline model, FALDA still provides measurable improvements (e.g., 14.6\% MSE reduction on Exchange) while maintaining competitive performance across other datasets. Notably, D$^3$U exhibits performance degradation when using iTransformer as the backbone, as evidenced in Tables 1 and 6 of \cite{2025D3U}.
These empirical results validate that FALDA effectively enhances forecasting performance for both relatively weaker backbones and state-of-the-art backbones, demonstrating its general applicability in time series forecasting tasks.

\begin{table*}[htbp]
\vspace{-1pt}
\centering
\setlength{\tabcolsep}{3pt} 
\renewcommand{\arraystretch}{1.1} 
\caption{Plug-and-play performance improvement of FALDA on existing point forecasting methods. Better values are highlighted in \textbf{bold}.}
\vspace{1pt}
\begin{tiny} 
\resizebox{0.65 \linewidth}{!}{ 
\begin{tabular}{c|cc|cc|cc|cc} 
\toprule[1.2pt]
\multirow{2}{*}{\scalebox{0.9}{Model}} & \multicolumn{2}{c}{\scalebox{0.9}{Exchange}} & \multicolumn{2}{c}{\scalebox{0.9}{ILI}} & \multicolumn{2}{c}{\scalebox{0.9}{ETTm2}} & \multicolumn{2}{c}{\scalebox{0.9}{Electricity}} \\ 
\cmidrule(lr){2-3}\cmidrule(lr){4-5}\cmidrule(lr){6-7}\cmidrule(lr){8-9}
\scalebox{0.9}{Metric} & \scalebox{0.9}{MSE} & \scalebox{0.9}{MAE} & \scalebox{0.9}{MSE} & \scalebox{0.9}{MAE} & \scalebox{0.9}{MSE} & \scalebox{0.9}{MAE} & \scalebox{0.9}{MSE} & \scalebox{0.9}{MAE} \\ 
\toprule[1.2pt]
\scalebox{0.9}{Autoformer} & 0.537 & 0.526 & 3.366 & 1.210 & 0.269 & 0.327 & 0.227 & 0.332 \\ 
{+ ours} 
\scalebox{0.9}{} & \textbf{0.232} & \textbf{0.351} & \textbf{2.655} & \textbf{1.118} & \textbf{0.247} & \textbf{0.313} & \textbf{0.209} & \textbf{0.316} \\ 
\scalebox{0.9}{Promotion} & \textbf{56.7\%} & \textbf{33.3\%} & \textbf{21.1\%} & \textbf{7.5\%} & \textbf{8.2\%} & \textbf{4.2\%} & \textbf{7.6\%} & \textbf{4.7\%} \\ \midrule
\scalebox{0.9}{Informer} & 1.092 & 0.853 & 4.620 & 1.456 & 0.494 & 0.525 & 0.319 & 0.399 \\ 
{+ ours} 
\scalebox{0.9}{} & \textbf{0.367} & \textbf{0.460} & \textbf{3.122} & \textbf{1.178} & \textbf{0.293} & \textbf{0.363} & \textbf{0.305} & \textbf{0.388} \\ 
\scalebox{0.9}{Promotion} & \textbf{66.4\%} & \textbf{46.2\%} & \textbf{32.4\%} & \textbf{19.1\%} & \textbf{40.8\%} & \textbf{30.9\%} & \textbf{4.5\%} & \textbf{2.8\%} \\ \midrule
\scalebox{0.9}{Transformer} & 0.975 & 0.765 & 4.044 & 1.327 & 0.427 & 0.472 & 0.256 & 0.347 \\ 
{+ ours} 
\scalebox{0.9}{} & \textbf{0.403} & \textbf{0.488} & \textbf{3.226} & \textbf{1.254} & \textbf{0.390} & \textbf{0.423} & \textbf{0.251} & \textbf{0.344} \\ 
\scalebox{0.9}{Promotion} & \textbf{58.7\%} & \textbf{36.3\%} & \textbf{20.2\%} & \textbf{5.5\%} & \textbf{8.7\%} & \textbf{10.2\%} & \textbf{1.8\%} & \textbf{0.9\%} \\ \midrule
\scalebox{0.9}{iTransformer} & 0.193 & 0.315 & 1.833 & 0.828 & 0.246 & \textbf{0.300} & 0.164 & 0.248 \\ 
{+ ours} 
\scalebox{0.9}{} & \textbf{0.165} & \textbf{0.296} & \textbf{1.666} & \textbf{0.821} & \textbf{0.246} & 0.301 & \textbf{0.163} & \textbf{0.248} \\ 
\scalebox{0.9}{Promotion} & \textbf{14.6\%} & \textbf{6.0\%} & \textbf{9.1\%} & \textbf{0.8\%} & \textbf{0.1\%} & {-0.5\%} & \textbf{1.1\%} & \textbf{0.0\%} \\ \bottomrule[1.2pt]
\end{tabular}}
\end{tiny}
\label{tab:plug-and-play}
\vspace{-5pt}
\end{table*}

\subsection{Ablation Study}
To further validate that our architecture enables the diffusion model to focus on aleatoric uncertainty learning, we investigate the model's performance under different conditioning strategies. Table~\ref{tab:ablation:condition} compares the results when using $X_{\text{noise}}$, $X$ as conditioning inputs, along with an unconditional case. The experiments show that the $X_{\text{noise}}$-conditioned version achieves optimal performance across all evaluated datasets, while the unconditional case performs comparably to the $X_{\text{noise}}$-conditioned scenario. In contrast, the $X$-conditioned approach shows the worst performance among the three conditioning types. These results indicate that 
epistemic uncertainty does not dominate the components of diffusion learning, thereby the residual estimation through $X$-conditioning provides limited benefits. In conclusion, the FALDA framework successfully achieves enhanced learning of aleatoric uncertainty while simultaneously improving point estimation capability. 
Additionally, Appendix~\ref{appendix:ab_denoiser} presents an ablation study comparing DEMA with its variants, systematically validating the effectiveness of its time-decomposition operation.
Appendix \ref{appendix:ab_training_strategy} shows the impact of different fine-tuning strategies during training.
Appendix~\ref{appendix:ab_diffusion} demonstrates the effectiveness of the DMRR component and the NS-Adapter module. Figure \ref{fig:hexagonal_performance_falda} shows the advantage of our framework when using the same NSformer \cite{2022NSformer} backbone. The complete experimental results are provided in Appendix \ref{appendix:exp:same backbone}.
\begin{table}[h]
\vspace{-5pt}
\centering
\caption{Ablation study on different condition strategies. The best results are boldfaced.}
\begin{footnotesize}
\resizebox{0.75 \linewidth}{!}{
\begin{tabular}{c|cc|cc|cc|cc}
\toprule
\multirow{2}{*}{Condition type} & \multicolumn{2}{c}{Exchange} & \multicolumn{2}{c}{ILI} & \multicolumn{2}{c}{ETTm2} & \multicolumn{2}{c}{Weather} \\
\cmidrule(lr){2-3}\cmidrule(lr){4-5}\cmidrule(lr){6-7}\cmidrule(lr){8-9}
 & MSE & MAE & MSE & MAE & MSE & MAE & MSE & MAE \\
\midrule
$X_{\text{noise}}$ & \textbf{0.165} & \textbf{0.296} & \textbf{1.666} & 0.821 & \textbf{0.246} & \textbf{0.301} & \textbf{0.215} & \textbf{0.255} \\
uncond & 0.184 & 0.311 & 1.675 & \textbf{0.785} & 0.251 & 0.307 & 0.217 & 0.260 \\
$X$ & 0.178 & 0.312 & 1.994 & 0.966 & 0.258 & 0.313 & 0.216 & 0.261 \\
\bottomrule
\end{tabular}}
\end{footnotesize}
\label{tab:ablation:condition}
\vspace{-10pt}
\end{table} 

\section{Related Works}
Diffusion-based time series forecasting models have demonstrated their efficacy in modeling multivariate time series distributions.  
TimeGrad \cite{rasul2021timegrad} integrates a recurrent neural network (RNN) with a diffusion model for autoregressive forecasting, using hidden states to condition the diffusion process. While effective for short-term predictions, its autoregressive nature causes error accumulation and inefficiency in long-term forecasting.
CSDI \cite{tashiro2021csdi} adopts a non-autoregressive fashion which uses self-supervised masking to guide the denoising process, with historical information and observation as conditions. 
SSSD \cite{sssd} enhances time series modeling by integrating conditional diffusion with structured state space models, improving long-range dependency capture and computational efficiency over transformer-based approaches like CSDI.
TimeDiff \cite{shen2023timediff} introduces inductive bias to the outputs of the conditioning network through two mechanisms (future mixup and autoregressive initialization) to facilitate the denoising process. 
A range of diffusion-based TSF models, combined with strong point forecasting models, have recently demonstrated strong performance in both point forecasting and probabilistic forecasting capability. TMDM \cite{2024TMDM} utilizes strong point forecasting models, such as NSformer \cite{2022NSformer}, to extract prior knowledge and inject it into both the forward and reverse diffusion processes to guide the generation of the forecast window. 
D$^3$U \cite{2025D3U} employs a point forecasting model to nonprobabilistically model high-certainty components in the time series, generating embedded representations that are conditionally injected into a diffusion model.

\section{Conclusion}
In this paper, we present \textbf{FALDA}, a Fourier-based diffusion framework for time series forecasting that systematically addresses both deterministic patterns and stochastic uncertainties. Our Fourier decomposition and component-specific modeling approach enable FALDA to decouple complex time series into interpretable components while clearly separating epistemic and aleatoric uncertainty. The integration of a conditional diffusion model with historical noise conditioning significantly improves stochastic component prediction.
Our theoretical analysis provides formal guarantees for the mathematical foundations of the framework, while the proposed alternating training strategy is proven effective for jointly optimizing multiple model components. Extensive empirical evaluations across six diverse real-world datasets consistently demonstrate FALDA's superior performance, setting new state-of-the-art results in both point forecasting and probabilistic prediction tasks.


\bibliographystyle{unsrt}
\bibliography{ref}

\begin{thebibliography}{10}

\bibitem{2020Multivariatefinance}
Hui Li, Yunpeng Cui, Shuo Wang, Juan Liu, and Yilin Yang.
\newblock Multivariate financial time-series prediction with certified robustness.
\newblock {\em IEEE Access}, 2020.

\bibitem{2023Medical}
Sven Festag and C.~Spreckelsen.
\newblock Medical multivariate time series imputation and forecasting based on a recurrent conditional wasserstein gan and attention.
\newblock {\em Journal of Biomedical Informatics}, 2023.

\bibitem{lv2014traffic}
Yisheng Lv, Yanjie Duan, Wenwen Kang, Zhengxi Li, and Fei-Yue Wang.
\newblock Traffic flow prediction with big data: A deep learning approach.
\newblock {\em IEEE Transactions on Intelligent Transportation Systems}, 2014.

\bibitem{dai2020hybrid}
Rui Dai, Shenkun Xu, Qian Gu, Chenguang Ji, and Kaikui Liu.
\newblock Hybrid spatio-temporal graph convolutional network: Improving traffic prediction with navigation data.
\newblock In {\em Proceedings of the 26th ACM SIGKDD Conference on Knowledge Discovery and Data Mining}, 2020.

\bibitem{wu2021autoformer}
Haixu Wu, Jiehui Xu, Jianmin Wang, and Mingsheng Long.
\newblock Autoformer: Decomposition transformers with auto-correlation for long-term series forecasting.
\newblock In {\em Conference on Neural Information Processing Systems}, 2021.

\bibitem{2023DLinear}
Ailing Zeng, Muxi Chen, Lei Zhang, and Qiang Xu.
\newblock Are transformers effective for time series forecasting?
\newblock In {\em Proceedings of the AAAI conference on artificial intelligence}, 2023.

\bibitem{liu2024itransformer}
Yong Liu, Tengge Hu, Haoran Zhang, Haixu Wu, Shiyu Wang, Lintao Ma, and Mingsheng Long.
\newblock itransformer: Inverted transformers are effective for time series forecasting.
\newblock In {\em The Twelfth International Conference on Learning Representations}, 2024.

\bibitem{esser2024scaling}
Patrick Esser, Sumith Kulal, Andreas Blattmann, Rahim Entezari, Jonas M{\"u}ller, Harry Saini, Yam Levi, Dominik Lorenz, Axel Sauer, Frederic Boesel, et~al.
\newblock Scaling rectified flow transformers for high-resolution image synthesis.
\newblock In {\em Forty-first international conference on machine learning}, 2024.

\bibitem{rombach2022high}
Robin Rombach, Andreas Blattmann, Dominik Lorenz, Patrick Esser, and Bj{\"o}rn Ommer.
\newblock High-resolution image synthesis with latent diffusion models.
\newblock In {\em Proceedings of the IEEE/CVF conference on computer vision and pattern recognition}, pages 10684--10695, 2022.

\bibitem{peebles2023scalable}
William Peebles and Saining Xie.
\newblock Scalable diffusion models with transformers.
\newblock In {\em Proceedings of the IEEE/CVF international conference on computer vision}, pages 4195--4205, 2023.

\bibitem{chu2024visionllama}
Xiangxiang Chu, Jianlin Su, Bo~Zhang, and Chunhua Shen.
\newblock Visionllama: A unified llama backbone for vision tasks.
\newblock In {\em European Conference on Computer Vision}, pages 1--18. Springer, 2024.

\bibitem{liuflow}
Xingchao Liu, Chengyue Gong, et~al.
\newblock Flow straight and fast: Learning to generate and transfer data with rectified flow.
\newblock In {\em The Eleventh International Conference on Learning Representations}, 2023.

\bibitem{lan2025flux}
Rui Lan, Yancheng Bai, Xu~Duan, Mingxing Li, Lei Sun, and Xiangxiang Chu.
\newblock Flux-text: A simple and advanced diffusion transformer baseline for scene text editing.
\newblock {\em arXiv preprint arXiv:2505.03329}, 2025.

\bibitem{ramesh2021zero}
Aditya Ramesh, Mikhail Pavlov, Gabriel Goh, Scott Gray, Chelsea Voss, Alec Radford, Mark Chen, and Ilya Sutskever.
\newblock Zero-shot text-to-image generation.
\newblock In {\em International conference on machine learning}, pages 8821--8831. Pmlr, 2021.

\bibitem{flux2024}
Black~Forest Labs.
\newblock Flux.
\newblock \url{https://github.com/black-forest-labs/flux}, 2024.

\bibitem{chu2025usp}
Xiangxiang Chu, Renda Li, and Yong Wang.
\newblock Usp: Unified self-supervised pretraining for image generation and understanding.
\newblock {\em arXiv preprint arXiv:2503.06132}, 2025.

\bibitem{zhang2024trip}
Zhongwei Zhang, Fuchen Long, Yingwei Pan, Zhaofan Qiu, Ting Yao, Yang Cao, and Tao Mei.
\newblock {TRIP: Temporal Residual Learning with Image Noise Prior for Image-to-Video Diffusion Models}.
\newblock In {\em Proceedings of the IEEE/CVF Conference on Computer Vision and Pattern Recognition}, 2024.

\bibitem{2025motionpro}
Zhongwei Zhang, Fuchen Long, Zhaofan Qiu, Yingwei Pan, Wu~Liu, Ting Yao, and Tao Mei.
\newblock {MotionPro: A Precise Motion Controller for Image-to-Video Generation}.
\newblock In {\em Proceedings of the IEEE/CVF Conference on Computer Vision and Pattern Recognition}, 2025.

\bibitem{zheng2024open}
Zangwei Zheng, Xiangyu Peng, Tianji Yang, Chenhui Shen, Shenggui Li, Hongxin Liu, Yukun Zhou, Tianyi Li, and Yang You.
\newblock Open-sora: Democratizing efficient video production for all.
\newblock {\em arXiv preprint arXiv:2412.20404}, 2024.

\bibitem{bar2024lumiere}
Omer Bar-Tal, Hila Chefer, Omer Tov, Charles Herrmann, Roni Paiss, Shiran Zada, Ariel Ephrat, Junhwa Hur, Guanghui Liu, Amit Raj, et~al.
\newblock Lumiere: A space-time diffusion model for video generation.
\newblock In {\em SIGGRAPH Asia 2024 Conference Papers}, pages 1--11, 2024.

\bibitem{hu2024animate}
Li~Hu.
\newblock Animate anyone: Consistent and controllable image-to-video synthesis for character animation.
\newblock In {\em Proceedings of the IEEE/CVF Conference on Computer Vision and Pattern Recognition}, pages 8153--8163, 2024.

\bibitem{blattmann2023stable}
Andreas Blattmann, Tim Dockhorn, Sumith Kulal, Daniel Mendelevitch, Maciej Kilian, Dominik Lorenz, Yam Levi, Zion English, Vikram Voleti, Adam Letts, et~al.
\newblock Stable video diffusion: Scaling latent video diffusion models to large datasets.
\newblock {\em arXiv preprint arXiv:2311.15127}, 2023.

\bibitem{yang2024cogvideox}
Zhuoyi Yang, Jiayan Teng, Wendi Zheng, Ming Ding, Shiyu Huang, Jiazheng Xu, Yuanming Yang, Wenyi Hong, Xiaohan Zhang, Guanyu Feng, et~al.
\newblock Cogvideox: Text-to-video diffusion models with an expert transformer.
\newblock {\em arXiv preprint arXiv:2408.06072}, 2024.

\bibitem{lin2024open}
Bin Lin, Yunyang Ge, Xinhua Cheng, Zongjian Li, Bin Zhu, Shaodong Wang, Xianyi He, Yang Ye, Shenghai Yuan, Liuhan Chen, et~al.
\newblock Open-sora plan: Open-source large video generation model.
\newblock {\em arXiv preprint arXiv:2412.00131}, 2024.

\bibitem{2024mgtsd}
Xinyao Fan, Yueying Wu, Chang Xu, Yuhao Huang, Weiqing Liu, and Jiang Bian.
\newblock {MG-TSD:} multi-granularity time series diffusion models with guided learning process.
\newblock In {\em International Conference on Learning Representations}, 2024.

\bibitem{tashiro2021csdi}
Yusuke Tashiro, Jiaming Song, Yang Song, and Stefano Ermon.
\newblock Csdi: Conditional score-based diffusion models for probabilistic time series imputation.
\newblock {\em Advances in neural information processing systems}, 2021.

\bibitem{2024mrdiff}
Lifeng Shen, Weiyu Chen, and James~T. Kwok.
\newblock Multi-resolution diffusion models for time series forecasting.
\newblock In {\em International Conference on Learning Representations}, 2024.

\bibitem{2024series_to_series}
Hao Yang, Zhanbo Feng, Feng Zhou, Robert~C Qiu, and Zenan Ling.
\newblock Series-to-series diffusion bridge model.
\newblock {\em arXiv preprint arXiv:2411.04491}, 2024.

\bibitem{24diffusionts}
Xinyu Yuan and Yan Qiao.
\newblock Diffusion-ts: Interpretable diffusion for general time series generation.
\newblock In {\em International Conference on Learning Representations}, 2024.

\bibitem{2022NSformer}
Yong Liu, Haixu Wu, Jianmin Wang, and Mingsheng Long.
\newblock Non-stationary transformers: Exploring the stationarity in time series forecasting.
\newblock In {\em Conference on Neural Information Processing Systems}, 2022.

\bibitem{2024fan}
Weiwei Ye, Songgaojun Deng, Qiaosha Zou, and Ning Gui.
\newblock Frequency adaptive normalization for non-stationary time series forecasting.
\newblock In {\em Conference on Neural Information Processing Systems}, 2024.

\bibitem{2024TMDM}
Yuxin Li, Wenchao Chen, Xinyue Hu, Bo~Chen, Baolin Sun, and Mingyuan Zhou.
\newblock Transformer-modulated diffusion models for probabilistic multivariate time series forecasting.
\newblock In {\em International Conference on Learning Representations}, 2024.

\bibitem{2025D3U}
Qi~Li, Zhenyu Zhang, Lei Yao, Zhaoxia Li, Tianyi Zhong, and Yong Zhang.
\newblock Diffusion-based decoupled deterministic and uncertain framework for probabilistic multivariate time series forecasting.
\newblock In {\em International Conference on Learning Representations}, 2025.

\bibitem{rasul2021timegrad}
Kashif Rasul, Calvin Seward, Ingmar Schuster, and Roland Vollgraf.
\newblock Autoregressive denoising diffusion models for multivariate probabilistic time series forecasting.
\newblock In {\em International Conference on Machine Learning}, 2021.

\bibitem{2020ddpm}
Jonathan Ho, Ajay Jain, and Pieter Abbeel.
\newblock Denoising diffusion probabilistic models.
\newblock In {\em Conference on Neural Information Processing Systems}, 2020.

\bibitem{2022card}
Xizewen Han, Huangjie Zheng, and Mingyuan Zhou.
\newblock {CARD:} classification and regression diffusion models.
\newblock In {\em Conference on Neural Information Processing Systems}, 2022.

\bibitem{gawlikowski2023aleatoric}
Jakob Gawlikowski, Cedrique Rovile~Njieutcheu Tassi, Mohsin Ali, Jongseok Lee, Matthias Humt, Jianxiang Feng, Anna Kruspe, Rudolph Triebel, Peter Jung, Ribana Roscher, et~al.
\newblock A survey of uncertainty in deep neural networks.
\newblock {\em Artificial Intelligence Review}, 2023.

\bibitem{shen2023timediff}
Lifeng Shen and James Kwok.
\newblock Non-autoregressive conditional diffusion models for time series prediction.
\newblock In {\em International Conference on Machine Learning}, 2023.

\bibitem{2021ddim}
Jiaming Song, Chenlin Meng, and Stefano Ermon.
\newblock Denoising diffusion implicit models.
\newblock In {\em International Conference on Learning Representations}, 2021.

\bibitem{2017transformer}
Ashish Vaswani, Noam Shazeer, Niki Parmar, Jakob Uszkoreit, Llion Jones, Aidan~N Gomez, \L~ukasz Kaiser, and Illia Polosukhin.
\newblock Attention is all you need.
\newblock In {\em Advances in Neural Information Processing Systems}, 2017.

\bibitem{li2024mar}
Tianhong Li, Yonglong Tian, He~Li, Mingyang Deng, and Kaiming He.
\newblock Autoregressive image generation without vector quantization.
\newblock {\em Advances in Neural Information Processing Systems}, 2024.

\bibitem{zhou2021informer}
Haoyi Zhou, Shanghang Zhang, Jieqi Peng, Shuai Zhang, Jianxin Li, Hui Xiong, and Wancai Zhang.
\newblock Informer: Beyond efficient transformer for long sequence time-series forecasting.
\newblock In {\em Proceedings of the AAAI Conference on Artificial Intelligence}, 2021.

\bibitem{2022FEDformer}
Tian Zhou, Ziqing Ma, Qingsong Wen, Xue Wang, Liang Sun, and Rong Jin.
\newblock Fedformer: Frequency enhanced decomposed transformer for long-term series forecasting.
\newblock In {\em {International Conference on Machine Learning}}, 2022.

\bibitem{wu2023timesnet}
Haixu Wu, Tengge Hu, Yong Liu, Hang Zhou, Jianmin Wang, and Mingsheng Long.
\newblock Timesnet: Temporal 2d-variation modeling for general time series analysis.
\newblock In {\em International Conference on Learning Representations}, 2023.

\bibitem{2023patchtst}
Yuqi Nie, Nam~H. Nguyen, Phanwadee Sinthong, and Jayant Kalagnanam.
\newblock A time series is worth 64 words: Long-term forecasting with transformers.
\newblock In {\em International Conference on Learning Representations}, 2023.

\bibitem{sssd}
Juan~Lopez Alcaraz and Nils Strodthoff.
\newblock Diffusion-based time series imputation and forecasting with structured state space models.
\newblock {\em Transactions on Machine Learning Research}, 2023.

\bibitem{2014vae}
Diederik~P. Kingma and Max Welling.
\newblock Auto-encoding variational bayes.
\newblock In {\em International Conference on Learning Representations}, 2014.

\bibitem{2017uncertainty}
Alex Kendall and Yarin Gal.
\newblock What uncertainties do we need in bayesian deep learning for computer vision?
\newblock In {\em Advances in Neural Information Processing Systems}, 2017.

\bibitem{Lai_Chang_Yang_Liu_2018}
Guokun Lai, Wei-Cheng Chang, Yiming Yang, and Hanxiao Liu.
\newblock Modeling long- and short-term temporal patterns with deep neural networks.
\newblock In {\em The 41st International ACM SIGIR Conference on Research \&amp; Development in Information Retrieval}, 2018.

\bibitem{li2019enhancing}
Shiyang Li, Xiaoyong Jin, Yao Xuan, Xiyou Zhou, Wenhu Chen, Yu-Xiang Wang, and Xifeng Yan.
\newblock Enhancing the locality and breaking the memory bottleneck of transformer on time series forecasting.
\newblock In {\em Conference on Neural Information Processing Systems}, 2019.

\bibitem{matheson1976scoring}
James~E. Matheson and Robert~L. Winkler.
\newblock Scoring rules for continuous probability distributions.
\newblock {\em Management Science}, 1976.

\bibitem{gneiting2007strictly}
Tilmann Gneiting and Adrian~E Raftery.
\newblock Strictly proper scoring rules, prediction, and estimation.
\newblock {\em Journal of the American statistical Association}, 2007.

\bibitem{yao2019quality}
Jiayu Yao, Weiwei Pan, Soumya Ghosh, and Finale Doshi-Velez.
\newblock Quality of uncertainty quantification for bayesian neural network inference.
\newblock {\em arXiv preprint arXiv:1906.09686}, 2019.

\bibitem{paszke2019pytorch}
Adam Paszke, Sam Gross, Francisco Massa, Adam Lerer, James Bradbury, Gregory Chanan, Trevor Killeen, Zeming Lin, Natalia Gimelshein, Luca Antiga, et~al.
\newblock Pytorch: An imperative style, high-performance deep learning library.
\newblock In {\em Conference on Neural Information Processing Systems}, 2019.

\bibitem{kingma2014adam}
Diederik~P Kingma and Jimmy Ba.
\newblock Adam: A method for stochastic optimization.
\newblock {\em arXiv preprint arXiv:1412.6980}, 2014.

\end{thebibliography}


\appendix

\section{Mathematical Derivations} \label{appendix:mathematical derivations}
\subsection{Preliminary: Denoising Diffusion Probabilistic Models}\label{appendix:ddpm}
Denoising Diffusion Probabilistic Models (DDPM) \cite{2020ddpm} is a canonical diffusion model consisting of the forward and reverse processes. Let $q(y_0)$ be the data distribution, the forward process is a Markov chain $\{y_0, y_1, ..., y_k, ...\}$ that gradually transforms the data distribution into a standard Gaussian distribution: $y_k \xrightarrow{d} \mathcal{N}(0, I), k\rightarrow \infty$. Here "$\xrightarrow{d}$" denotes convergence in distribution. The transition probability is  $q(y_{k} | y_{k-1}) = \mathcal{N}(y_{k}; \sqrt{\alpha_k}y_{k-1}, \beta_k I)$. where $\alpha_k = 1-\beta_k \in (0, 1)$ represents the noise schedule.  
The single-step transition formulation at step $k$ can be demonstrated as below using the reparameterization trick \cite{2014vae}:
\begin{equation}
    y_{k} = \sqrt{\alpha_k} y_{k-1} + \sqrt{\beta_k} z_{k} , \quad z_{k} \sim \mathcal{N}(0, I). 
\end{equation}
Iterating the single-step formulation leads to the multi-step transition formulation at step $k$:
\begin{equation} \label{eq:ddpm: multi-step}
 y_{k} = \sqrt{\bar{\alpha}_k} y_{0} + \sqrt{1-\bar{\alpha}_k} \bar{z}_{k}, \quad \bar{z}_{k} \sim \mathcal{N}(0, I).
\end{equation}
Here $ \bar{\alpha}_k = \prod_{s=1}^{k} \alpha_s, \bar{\beta}_k = \prod_{s=1}^{k} \beta_s $. The reverse process starts from a standard Gaussian noise $y_K$, and has the following posterior distribution at step $k$:
\begin{equation}\label{eq:ddpm_posterior}
\begin{gathered}
    q(y_{k-1} | y_k, y_0) = \mathcal{N}(y_{k-1}; \tilde{\mu}_k, \tilde{\beta}_kI), \\
    \tilde{\mu}_k = \frac{\sqrt{\bar{\alpha}_{k-1}}\beta_k}{1-\bar{\alpha}_k}y_0 + \frac{\sqrt{\alpha_k}(1-\bar{\alpha}_{k-1})}{1-\bar{\alpha}_k}y_k, \quad \tilde{\beta}_k = \frac{1-\bar{\alpha}_{k-1}}{1-\bar{\alpha}_k}\beta_k.
\end{gathered}
\end{equation}
By substituting $y_0$ with $y_0 = \frac{1}{\sqrt{\bar{\alpha}_k}}y_k - \frac{\sqrt{1-\bar{\alpha}_k}}{\sqrt{\bar{\alpha}_k}}\bar{z}_k$, we have $\tilde{\mu}(y_k, k) = \frac{1}{\sqrt{\alpha_k}}\left(y_k - \frac{\beta_k}{\sqrt{1-\bar{\alpha}_k}}\bar{z}_k\right)$. The mean $\tilde{\mu}_k$ is typically parameterized using two different strategies: (1) modeling the diffusion noise $\tilde{z}_k$ with $\hat{\epsilon}_{\theta}(y_k, k)$, or (2) directly parameterizing the target $y_0$ in Eq. \eqref{eq:ddpm_posterior} with $\hat{y}_{\theta}(y_k, k)$.

\subsection{Equivalence between CARD and DMRR}
\begin{prop}
Let $y_k$ be the Markov chain defined in Eq. \eqref{eq:card_forward}. Let $l_k = y_k - f_{\phi}(x)$, we have: 
\begin{equation}
q(l_{k} | l_{k-1}) = \mathcal{N}(l_{k}; \sqrt{\alpha_k}l_{k-1}, \beta_k I)
\label{eq:res_t-forward}
\end{equation}
and
\begin{equation} \label{eq:res-t-posterior}
\begin{gathered}
q(l_{k-1} | l_k, l_0) = \mathcal{N}(y_{l-1}; \tilde{\mu}_k, \tilde{\beta}_kI), \\
\tilde{\mu}_k = \frac{\sqrt{\bar{\alpha}_{k-1}}\beta_k}{1-\bar{\alpha}_k}l_0 + \frac{\sqrt{\alpha_k}(1-\bar{\alpha}_{k-1})}{1-\bar{\alpha}_k}l_k, \quad \tilde{\beta}_k = \frac{1-\bar{\alpha}_{k-1}}{1-\bar{\alpha}_k}\beta_k.
\end{gathered}
\end{equation}

Thus, the residual process $l_t$ exhibits identical Markovian dynamics to the standard DDPM framework in both forward and reverse processes as shown in Eq. \eqref{eq:ddpm: multi-step} and Eq. \eqref{eq:ddpm_posterior}.
\end{prop}

\begin{proof}
    \textbf{Proof of Equation \eqref{eq:res_t-forward}:} \\
    Starting from the result in Eq. \eqref{eq:card_forward}, 
    \begin{align*}
    l_k &=  y_k - f_{\phi}(x) \\
     &= \sqrt{\bar{\alpha}_k}y_0 + (1-\sqrt{\bar{\alpha}_k})f_{\phi}(x) + \sqrt{1-\bar{\alpha}_k}\bar{z}_k - f_{\phi}(x) \\
    &= \sqrt{\bar{\alpha}_k} (y_0 - f_{\phi}(x)) + \sqrt{1-\bar{\alpha}_k}\bar{z}_k + f_{\phi}(x) - f_{\phi}(x) \\
    &=\sqrt{\alpha_k} l_0 + \sqrt{1-\bar{\alpha}_k}\bar{z}_k.
    \end{align*}
    This demonstrates that $l_t$ satisfies the standard DDPM forward process formulation.

    \textbf{Proof of Equation \eqref{eq:res-t-posterior}:} \\
    since $ l_k  =  y_k - f_{\phi}(x)$ and $q(y_{k-1}| y_k, y_0) = \mathcal{N}(y_{k-1}; \tilde{m}_k, \tilde{\beta}_k I)$, we have:

    \begin{align*}
    q(l_{k-1} | l_k, l_0) = \mathcal{N}(l_{k-1}, \tilde{m}_k -f_{\phi}(x), \tilde{\beta}_k I).
    \end{align*}
    
    We now analyze the mean $\tilde{m}_k - f_{\phi}(x)$. With the definition of $\tilde{m}_k$ in Eq. \ref{eq:card-posterior}, we have:
    \begin{equation*}
        \tilde{m}_k-f_\phi(x)=A_ky_0+B_ky_k+(C_k-1)f_\phi(x),
    \end{equation*}
    where the coefficients are:
    \begin{equation*}
        \begin{gathered}
            A_k:=\frac{\beta_k\sqrt{\bar{\alpha}_{k-1}}}{1-\bar{\alpha}_k},\quad B_k:=\frac{(1-\bar{\alpha}_{k-1})\sqrt{\alpha_k}}{1-\bar{\alpha}_k},\quad \\ C_k:=1+\frac{(\sqrt{\bar{\alpha}_k}-1)(\sqrt{\alpha_k}+\sqrt{\bar{\alpha}_{k-1}})}{1-\bar{\alpha}_k}.
        \end{gathered}
    \end{equation*}
    Substituting $y_k = l_k + f_{\phi}(x)$ yields:
    \begin{equation*}
        \tilde{m}_k-f_\phi(x)=A_kl_0+B_kl_k+(A_k+B_k+C_k-1)f_\phi(x).
    \end{equation*}
    In the following step, the coefficients of $f_{\phi}(x)$ can be expanded as:
    \begin{align*}
        A_k+B_k+C_k-1 &=\frac{\beta_k\sqrt{\bar{\alpha}_{k-1}}+(1-\bar{\alpha}_{k-1})\sqrt{\alpha_k}}{1-\bar{\alpha}_k}+\frac{(\sqrt{\bar{\alpha}_k}-1)(\sqrt{\alpha_k}+\sqrt{\bar{\alpha}_{k-1}})}{1-\bar{\alpha}_k} \\
        &= \frac
        {
        \sqrt{\bar{\alpha}_{k-1}} - \alpha_k \sqrt{\bar{\alpha}_{k-1}} - \sqrt{\alpha_k}\bar{\alpha}_{k-1} + \sqrt{\alpha_k \bar{\alpha}_{k}} 
         + \sqrt{\bar{\alpha}_k \bar{\alpha}}_{k-1} - \sqrt{\bar{\alpha}_{k-1}} }
         {1-\bar{\alpha}_k} \\
         & = \frac
        {
         - \alpha_k \sqrt{\bar{\alpha}_{k-1}} - \sqrt{\alpha_k}\bar{\alpha}_{k-1} + \sqrt{\alpha_k \bar{\alpha}_{k}} 
         + \sqrt{\bar{\alpha}_k \bar{\alpha}_{k-1}  }}
         {1-\bar{\alpha}_k}.
    \end{align*}
    Using the identity $\bar{\alpha}_k=\bar{\alpha}_{k-1}\alpha_k$, we have:
    \begin{equation*}
        A_k+B_k+C_k-1 = 0.
    \end{equation*}
    Therefore, the posterior mean $\tilde{m}_k - f_{\phi}(x)$ satisfies:
    \begin{align*}
        \tilde{m}_k - f_{\phi}(x)
        &= \frac{\beta_k \sqrt{\bar{\alpha}_{k-1}}}{1-\bar{\alpha}_k}l_0 + \frac{(1-\bar{\alpha}_{k-1})\sqrt{\alpha_k}}{1-\bar{\alpha}_k}l_k \\
        &= \tilde{\mu}_k.
    \end{align*}
    We have thus established that the reverse distribution of the residual process satisfies: $q(l_{k-1} | l_k, l_0) = \mathcal{N}(l_{k-1}; \tilde{\mu}_k, \tilde{\beta}_kI)$, This completes the proof of Eq. \eqref{eq:res-t-posterior}.
\end{proof}

\section{Methodology of TMDM and D\texorpdfstring{\textsuperscript{3}U}{D3U}}\label{appendix:tmdm-d3u}
In this section, we present the details of two previously developed diffusion-based time series forecasting methods: Transformer-Modulated  Diffusion Model (TMDM) \cite{2024TMDM} and Diffusion-based Decoupled Deterministic and Uncertain framework (D$^3$U) \cite{2025D3U}. The notation employed below is consistent with the notation used in Section \ref{sec:falda:problem-statement}.

\subsection{TMDM}
TMDM employs CARD as its underlying diffusion framework. Given a conditional information $\hat{Y}$, the end point of TMDM's diffusion process is:
\begin{equation}
    \lim_{k \rightarrow \infty} q(Y^{(k)}| \hat{Y}) = \mathcal{N}(\hat{Y}, I).
\end{equation}
Here $Y^{(k)}$ represents the noise sample of $Y$ at step $k$. With a noise schedule $\alpha_t$ and $\beta_t$ defined in Section \ref{sec:DMRR}, the forward process at step k can be defined as:
\begin{equation}
    q\left(Y^{(k)} | Y^{(k-1)}, \hat{Y}\right) \sim \mathcal{N}\left(\sqrt{\alpha_k} Y^{(k-1)} + (1 - \sqrt{1-\beta_k}) \hat{Y}, \beta_k I\right).
\end{equation}
The posterior distribution in the reverse diffusion process is:
\begin{equation}
    q\left(Y^{(k-1)} | Y^{(k)}, Y^{(0)}, \hat{Y}\right) \sim \mathcal{N}\left(Y^{(k-1)}; \tilde{m}_k, \tilde{\beta}_k I\right),
\end{equation}
where $\tilde{m}_k$ and $\tilde{\beta}_k$ are consistent with Eq. \eqref{eq:card-posterior}. Specifically, $\tilde{m}_k$ satisfies:
\begin{equation}
    \tilde{m}_k = \frac{\beta_k \sqrt{\bar{\alpha}_{k-1}}}{1-\bar{\alpha}_k}Y^{(0)} + \frac{(1-\bar{\alpha}_{k-1})\sqrt{\alpha_k}}{1-\bar{\alpha}_k}Y^{(k)} + (1 + \frac{(\sqrt{\bar{\alpha}_k} -1)(\sqrt{\alpha_k}+ \sqrt{\bar{\alpha}_{k-1}})}{1-\bar{\alpha}_k})\hat{Y}.
\end{equation}

\subsection{\texorpdfstring{D\textsuperscript{3}U}{D3U}} \label{appendix:d3u_methodology}
The D$^3$U framework builds upon the DMRR diffusion architecture. It employs a pretrained network $f_{\text{D}^3\text{U}}$ to generate preliminary estimates $\hat{Y}$, where the encoder embedding $f_{\text{enc}}(X)$ serves as the condition for the reverse diffusion process.

Defining the residual term $R = Y - \hat{Y}$, the forward diffusion process follows:
\begin{equation}
    q\left(R^{(k)} | R^{(k-1)}, \hat{R}\right) \sim \mathcal{N}\left(\sqrt{\alpha_k} R^{(k-1)}, \beta_k I\right).
\end{equation}

The posterior process is:
\begin{equation} \label{eq:appendix-d3u-posterior}
     q\left(R^{(k-1)} | R^{(k)}, R^{(0)}, f_{\text{enc}}(X)\right) \sim \mathcal{N}\left(R^{(k-1)}; \tilde{\mu}_k, \tilde{\beta}_k I\right).
\end{equation}
Here $\tilde{\mu}_k$ is consistent with Eq. \eqref{eq:ddpm_posterior}:
\begin{equation}
    \tilde{\mu}_k = \frac{\sqrt{\bar{\alpha}_{k-1}}\beta_k}{1-\bar{\alpha}_k}R^{(0)} + \frac{\sqrt{\alpha_k}(1-\bar{\alpha}_{k-1})}{1-\bar{\alpha}_k}R^{(k)}.
\end{equation}

\section{Probability view of residual component modeling} \label{appendix:probability-view}

As discussed in Section \ref{sec:falda}, D³U models epistemic uncertainty by conditioning on encoder outputs without intentionally decoupling it from temporal aleatoric uncertainty. This limits optimal performance scaling on more capable backbone models, which already exhibit low epistemic uncertainty. In this section, we provide a probabilistic analysis of different modeling approaches for time series forecasting. Specifically, Appendix \ref{appendix:general_situation} summarizes the general case, while Appendices \ref{appendix:d3u_situation} and \ref{appendix: our_situation} respectively analyze the probabilistic modeling of D³U and FALDA, highlighting their distinct learning objectives. We demonstrate how FALDA models both types of uncertainty through time-series components decomposition, allowing both deterministic and probabilistic models to focus on learning their respective components.

\subsection{General Situation} \label{appendix:general_situation}
In general, a time series $X$ can be decomposed into two components:
\begin{equation}
    X=X_{\text{nf}} + \epsilon_X,
\end{equation}
 where $X_{\text{nf}}$ is the ideal noise-free part (incorporating trend, seasonality, and other structured patterns), and $\epsilon_{X}$ denotes the inherent zero-mean noise in the time series data. Notably, in real-world scenarios, $\epsilon_X$ often follows complex non-Gaussian distributions. This canonical decomposition naturally extends to the forecasting target: $Y = Y_{nf} + \epsilon_Y $. To simplify the notation, in the following paragraphs, the subscripts for the noises only indicate which components they are associated with. The goal of the time series forecasting task is then to learn the conditional distribution: $P(Y|X)$. Conventionally, a deterministic function $f$ is employed to estimate the posterior expectation:
\begin{equation}
E(Y| X) = E(Y_{nf}| X) + \mathbb{E}\left(\epsilon_Y | f_{\phi}(X)\right) = E(Y_{nf}| X)  \approx f(X_{\text{nf}} + \epsilon_{X}).
\end{equation} \label{eq:appendix:regression_form}
This yields the following regression form for the prediction:
\begin{equation}
    Y = f(X_{\text{nf}} + \epsilon_{X}) + \epsilon_{X, Y}.
\end{equation}
In Equation \ref{eq:appendix:regression_form}, $\epsilon_{X,Y}$ comprises two distinct uncertainty components: aleatoric uncertainty stemming from inherent data randomness (specifically, the time series noise), and epistemic uncertainty arising from model estimation errors \cite{2017uncertainty}.

 Under ideal conditions where the point-estimation model perfectly captures $\mathbb{E}(Y|X)$, $\epsilon_{X,Y}$ would reduce to purely aleatoric uncertainty and become uncorrelated with $f(X)$, satisfying:
\begin{equation}
    \mathbb{E}\left(\epsilon_{X, Y} | f_{\phi}(X)\right) = 0.
\end{equation}
This implies the lookback window $X$ contains no additional information to improve point forecasts, resulting in $\epsilon_{X,Y} = \epsilon_Y$. However, in practice, point-estimation models rarely achieve this theoretical optimum, typically retaining some epistemic uncertainty. The subsequent discussion will examine how different time series forecasting models handle these distinct uncertainty components.

\subsection{D3U situation} \label{appendix:d3u_situation}
As established in Appendix~\ref{appendix:d3u_methodology}, the D$^3$U framework leverages the encoder-derived embedding representation $f_{\text{enc}}(X)$ as a conditioning mechanism for probabilistic residual learning, subsequent to the preliminary estimation $f(X)$. Formulated within the regression expression in the previous section, this approach specifically targets the conditional expectation $E(\epsilon_{X, Y}| f_{\text{enc}}(X))$, yielding:
\begin{equation} \label{eq:d3u_probability_view}
Y = f(X) + g(f_{\text{enc}}(X)) + \tilde{\epsilon}_{X, Y}.
\end{equation}
In this context, $\tilde{\epsilon}_{X, Y}$ denotes the total uncertainty of D$^3$U. Since the encoder of the point estimation model $f$ learns a good representation of the historical time series, $g(f_{\text{enc}}(X))$ can further model the epistemic uncertainty of $f(X)$. Comparing to $\epsilon_{X, Y}$, $\tilde{\epsilon}_{X, Y}$ may contain less
epistemic uncertainty. However, due to the predominance of predictions with epistemic uncertainty, this facilitation may diminish when the backbone model is sufficiently powerful. More importantly, since the true probabilistic component, uncertainty, is not explicitly separated, diffusion models may focus on epistemic uncertainty rather than uncertainty. This undifferentiated treatment ultimately constrains their probabilistic learning capability.

\subsection{Our situation} \label{appendix: our_situation}
To mitigate the epistemic uncertainty, first, we decompose the history time series into three parts $X = X_{\text{non}} + X_{\text{stat}} + X_{\text{noise}}$. Three models are jointly trained to forecast the whole future time series. Beyond the point-estimation model, we introduce an NS-adapter to improve modeling accuracy and reduce epistemic uncertainty, thereby alleviating part of the computational burden on the diffusion model. This architecture allows the diffusion model to concentrate solely on capturing aleatoric uncertainty, with the noise component $X_{\text{noise}}$ serving as the conditioning input for the diffusion process. The corresponding mathematical formulation is as follows:
\begin{equation}
Y = f_{\text{non}}(X_{\text{non}}) + f_{\text{stat}}(X_{\text{stat}}) + g_{\text{noise}}(X_{\text{noise}}) + \bar{\epsilon}_{X, Y}.
\end{equation}
Under this formulation, $\bar{\epsilon}_{X,Y}$ contains more aleatoric uncertainty, since explicit component separation effectively mitigates epistemic uncertainty. Compared to the expression $g(f_{\text{enc}}(X)) + \tilde {\epsilon}_{X,Y}$ in Eq. ~\eqref{eq:d3u_probability_view}, our approach shows superior properties. First, the composite term $g_{\text{noise}}(X_{\text{noise}}) + \bar{\epsilon}_{X,Y}$ is not dominated by epistemic uncertainty, since $f_{\text{non}}$ already takes into account most of the non-smooth patterns. Second, this decomposition allows the diffusion model to focus more effectively on capturing pure uncertainty without interference from the cognitive uncertainty component.

\section{Algorithms} \label{appendix:alg}
We formally present the complete algorithmic procedures of FALDA. Algorithm \ref{alg:falda-training} details the end-to-end training protocol with multi-task optimization. The corresponding inference procedure is specified in Algorithm \ref{alg:falda-inference}. 

\begin{algorithm}[ht]
\caption{FALDA Training Procedure}
\label{alg:falda-training}
\begin{algorithmic}[1]
\State \textbf{Require}: TS-backbone $g_\phi$, NS-adapter $f_w$, denoiser $\hat{R}^{(0)}_{\theta}$
\State \textbf{Hyperparameters}: Threshold $\delta$, period $\Delta$, $k'$, noise schedule: $\alpha_t, \beta_t$, max diffusion step $K$
\State \textbf{Input}: Lookback window $X \in \mathbb{R}^{T \times D}$, future ground truth $Y \in \mathbb{R}^{S \times D}$

\State Initialize the parameteres
\Repeat
    \State \textbf{Decomposition via Fourier Transform} \hfill $\triangleright$ Eq. \eqref{eq:decomposition}, \eqref{eq:y_norm}
    
    \State $X_{\text{non}}, X_{\text{stat}}, X_{\text{noise}} \gets X$
    \State $Y_{\text{non}}, Y_{\text{stat}}, Y_{\text{noise}} \gets Y$
    
    \State \textbf{Non-stationary \& Stationary Components modeling:}
    \State $\hat{Y}_{\text{non}} \gets f_w(X_{\text{non}})$ \hfill $\triangleright$ Eq.~\eqref{eq:non-norm-adapter}
    \State $\hat{Y}_{\text{stat}} \gets g_\phi(X_{\text{stat}})$
    
    \State \textbf{Residual Learning:}
    \State $R \gets Y - \hat{Y}_{\text{non}} - \hat{Y}_{\text{stat}}$
    \State $k \sim \mathcal{U}(\{1,2,..., K \})$
    \State$\epsilon \sim \mathcal{N}(0, I)$
    \State $R^{(k)} \gets \sqrt{\bar{\alpha}_k}R + \sqrt{1-\bar{\alpha}_k}\epsilon$, $R^{(k')} \gets \sqrt{\bar{\alpha}_{k'}}R + \sqrt{1-\bar{\alpha}_{k'}}\epsilon$,
    \State Predict residual: $\hat{R}^{(0)}_{\theta}(R^{(k)}, k, X_{\text{noise}}),  \hat{R}^{(0)}_{\theta}(R^{(k')}, k', X_{\text{noise}})$
    \State \textbf{Loss Computation:}
    \State Compute the loss $\mathcal{L}$ in Eq. \eqref{eq:loss}
    \State Take gradient descent step on: $\nabla \mathcal{L}$
\Until converged
\end{algorithmic}
\end{algorithm}

\begin{algorithm}[ht]
\caption{FALDA Inference Procedure}
\label{alg:falda-inference}
\begin{algorithmic}[1]

\State \textbf{Require}: Pretrained TS-backbone $g_\phi$, NS-adapter $f_w$ and denoiser $\hat{R}^{(0)}_{\theta}$
\State \textbf{Input}: Lookback window $X \in \mathbb{R}^{T \times D}$
\State \textbf{Decomposition via Fourier Transform:} \hfill $\triangleright$ Eq. \eqref{eq:decomposition}, \eqref{eq:y_norm}
\State $X_{\text{non}}, X_{\text{stat}}, X_{\text{noise}} \gets X$ 

\State \textbf{Predict Non-stationary \& Stationary Terms:}
\State $\hat{Y}_{\text{non}} \gets f_w(X_{\text{non}})$
\State $\hat{Y}_{\text{stat}} \gets g_\phi(X_{\text{stat}})$

\State \textbf{Generate Residual Prediction via Reverse Diffusion:}
\State Sample $R^{(K)} \sim \mathcal{N}(0, I)$
\For{$k = K$ \textbf{down to} $1$}
    \State Predict residual: $\hat{R}^{(0)} \gets \hat{R}^{(0)}_{\theta}(R^{(k)}, k, X_{\text{noise}})$
    \State Compute posterior mean 
    \State $\quad \tilde{\mu}_\theta \gets \frac{\sqrt{\bar{\alpha}_{k-1}}\beta_k}{1-\bar{\alpha}_k} \hat{R}^{(0)} + \frac{\sqrt{\alpha_k}(1-\bar{\alpha}_{k-1})}{1-\bar{\alpha}_k}R^{(k)}$
    \State Sample $R^{(k-1)} \sim \mathcal{N}(\tilde{\mu}_\theta, \tilde{\beta}_k I)$ \hfill $\triangleright$ Eq.~\eqref{eq:ddpm_posterior}
\EndFor
\State $\hat{R} \gets R^{(0)}$
\State \textbf{Final Prediction:}
\State $\hat{Y} \gets \hat{Y}_{\text{non}} + \hat{Y}_{\text{stat}} + \hat{R}$
\State \textbf{Return} $\hat{Y}$
\end{algorithmic}
\end{algorithm}

\section{Experiment Details}  \label{appendix:sec:experiment_details}
\subsection{Datasets} \label{appendix:dataset}
Experiments are performed on six widely-used real-world time series datasets: (1) influenza-like illness (ILI) reports the weekly ratio of patients presenting influenza-like symptoms to total clinical visits, obtained from U.S. CDC surveillance data from 2002 to 2021. \footnote{ILI: \url{https://gis.cdc.gov/grasp/fluview/fluportaldashboard.html}} Exchange-Rate \cite{Lai_Chang_Yang_Liu_2018} provides daily currency exchange rates for eight countries from 1990 to 2016. \footnote{Exchange: \url{https://github.com/laiguokun/multivariate-time-series-data}} ETTm2 \cite{zhou2021informer} contains 7 factors of electricity transformer from July 2016 to July 2018, which is recorded by 15 minutes. \footnote{ETTm2: \url{https://github.com/zhouhaoyi/ETDataset}} Electricity \cite{li2019enhancing} collects hourly power consumption from 321 customers from 2012 through 2014. \footnote{Electricity: \url{https://archive.ics.uci.edu/ml/datasets/ElectricityLoadDiagrams20112014}} Traffic \cite{wu2023timesnet} collates hourly road occupancy rates measured by 862 sensors on San Francisco Bay Area freeways between January 2015 and December 2016. \footnote{Traffic: \url{https://zenodo.org/record/4656132}} Weather \cite{zhou2021informer} includes meteorological time series collected from the Weather Station of the Max Planck Biogeochemistry Institute in 2020, with 21 meteorological indicators collected every 10 minutes. \footnote{Weather: \url{https://www.bgc-jena.mpg.de/wetter/}}

We follow the data processing protocol and split configurations from \cite{wu2021autoformer} and \cite{2024TMDM}. The lookback length is fixed as 96, and the prediction length is fixed as 192, with the exception of the ILI dataset, where the lookback length and prediction length are both set to 36. The details of all the datasets are provided in Table \ref{tab:dataset}.

\begin{table}[htbp] 
\centering
\small
\caption{Detailed dataset descriptions, including dimension, context length, label length, prediction length, and frequency.}
\begin{tabular}{lccccc}
\hline
Dataset & Dim  & Context length & Label length & Prediction length & Frequency \\
\hline
ILI & 7  & 36 & 16 & 36 & 1 week \\
Exchange & 8  & 96 & 48 & 192 & 1 day \\
Electricity & 321  & 96 & 48 & 192 & 1 hour \\
Traffic & 862  & 96 & 48 & 192 & 1 hour \\
ETTm2 & 7  & 96 & 48 & 192 & 15 mins \\
Weather & 21 & 96 & 48 & 192 & 10 mins \\
\hline
\end{tabular}
\label{tab:dataset}
\end{table}

\subsection{Evaluation Metrics} \label{appendix:metrics}

We employ two categories of evaluation metrics: deterministic metrics for point forecasts and probabilistic metrics for uncertainty estimation. Let $x \in \mathbb{R}^d$ denote the ground truth values and $\hat{x} \in \mathbb{R}^d$ represent the predicted values.

\begin{itemize}
    \item \textbf{Mean Squared Error (MSE)}:
    \begin{equation}
        \text{MSE}(x, \hat{x}) = \frac{1}{d} \|x - \hat{x}\|_2^2 = \frac{1}{d} \sum_{i=1}^d (x_i - \hat{x}_i)^2,
    \end{equation}
    where $\|\cdot\|_2$ denotes the $\ell_2$ norm.
    
    \item \textbf{Mean Absolute Error (MAE)}:
    \begin{equation}
        \text{MAE}(x, \hat{x}) = \frac{1}{d} \|x - \hat{x}\|_1 = \frac{1}{d} \sum_{i=1}^d |x_i - \hat{x}_i|,
    \end{equation}
    where $\|\cdot\|_1$ denotes the $\ell_1$ norm.
\end{itemize}

For assessing probabilistic forecasts and uncertainty estimation, we utilize:

\begin{itemize}
    \item \textbf{Continuous Ranked Probability Score (CRPS)} \cite{matheson1976scoring}, \cite{gneiting2007strictly}:
    \begin{equation}
        \text{CRPS}(F, x) = \int_{-\infty}^{\infty} (F(y) - \mathbb{I}\{x \leq y\})^2 dy,
    \end{equation}
    where $F(y)$ is the predicted cumulative distribution function.
    
    \item \textbf{Summed CRPS (CRPS$_{\text{sum}}$)}:
    \begin{equation}
        \text{CRPS}_{\text{sum}} = \mathbb{E}_t \left[\text{CRPS}(F^{-1}_{\text{sum}}, \textstyle\sum_{i=1}^d x_i)\right],
    \end{equation}
    where $F^{-1}_{\text{sum}}$ is obtained through dimension-wise summation of samples.
\end{itemize}

To specifically evaluate prediction intervals, we employ:

\begin{itemize}
    \item \textbf{Prediction Interval Coverage Probability (PICP)} \cite{yao2019quality}:
    \begin{equation}
        \text{PICP} = \frac{1}{N} \sum_{i=1}^N \mathbb{I}\{x_i \in [\hat{x}_i^{\text{low}}, \hat{x}_i^{\text{high}}]\},
    \end{equation}
    where $N$ represents the total number of observations, $x_i \in \mathbb{R}^d$ denotes the true value for the $i$-th observation, and $\hat{x}_n^{\text{low}}$ and $\hat{x}_n^{\text{high}}$ correspond to the $2.5^{th}$ and $97.5^{th}$ percentiles of the predicted distribution respectively, with $\mathbb{I}$ being the indicator function. This metric quantifies the empirical coverage probability by measuring the proportion of true observations falling within the predicted interval bounds. When the predicted distribution matches the true data distribution perfectly, the PICP should theoretically equal the nominal coverage level of 95\% for the specified $2.5^{th}-97.5^{th}$ percentile range.
    
    \item \textbf{Quantile Interval Coverage Error (QICE)} \cite{2022card}:
    \begin{equation}
        \text{QICE} = \frac{1}{M} \sum_{m=1}^{M} \left|\rho_m - \frac{1}{M}\right|, \quad 
        \rho_m = \frac{1}{N} \sum_{i=1}^N \mathbb{I}\{x_i \in [\hat{x}_i^{\text{low},m}, \hat{x}_i^{\text{high},m}]\}.
    \end{equation}
    QICE can be viewed as PICP with finer granularity and without uncovered quantile ranges. Under the optimal scenario where the predicted distribution perfectly matches the target distribution, the QICE value should be equal to 0.
\end{itemize}

\subsection{Implementation of Non-Stationary Adapter in FALDA} \label{appendix:model-non-stationary-adapter}
As discussed in Section~\ref{sec:falda}, we propose a non-stationary adapter $f_w$ to capture the non-stationary patterns in time series data. While a linear projection from $X_{\text{non}}$ to $\hat{Y}_{\text{non}}$ offers a straightforward approach, we enhance this design by additionally incorporating the complete lookback window $X$ as auxiliary input following the approach outlined in~\cite{2024fan}. This extension enables richer temporal context utilization, improving prediction accuracy for $Y_{\text{non}}$. The output of the adapter is computed as follows:

\begin{equation}
    \hat{Y}_{\text{non}} = f_w(X_{\text{non}}, X) = W_3 \operatorname{ReLU}\left(W_2 \operatorname{Concat}\left(\operatorname{ReLU}(W_1 X_{\text{non}}), X\right)\right),
\end{equation}

where $W_1$, $W_2$, and $W_3$ are learnable weight matrices. The concatenation operation explicitly combines the processed non-stationary features with the original input, allowing the network to leverage both representations.

\subsection{Implementation details} \label{appendix:implementation}
All the experiments are conducted on a single NVIDIA L20 48GB GPU, utilizing PyTorch \cite{paszke2019pytorch}.
We set the number of diffusion steps to $K=1000$, adopting a linear noise schedule following the configuration in \cite{2024TMDM}. Following DDIM \cite{2021ddim}, we accelerate the sampling procedure by selecting a 10-point subsequence (with a stride of 100 steps) from the original 1000 diffusion steps, effectively skipping intermediate computations while maintaining generation quality. Correspondingly, we adjust the fine-tuning diffusion step $k'$ to align with the subsampling stride, setting $k' = 100$ to match the first sampling interval. The parameter $\eta$ controls the determinism level in DDIM sampling, where $\eta=0$ yields a fully deterministic generation process. We utilize the Adam optimizer \cite{kingma2014adam} with a learning rate of $10^{-4}$ and L1 loss. Early stopping is applied after \{5, 10, 15\} epochs without improvement, with a maximum of 200 epochs.
 The batch size is set to 32 during training and 8 for testing. The context length, label length, and prediction length are detailed in Table \ref{tab:dataset}. To ensure robust statistical evaluation, we generate 100 prediction instances for each test sample to reliably compute the evaluation metrics. We show the point estimate performance and probabilistic forecasting performance in Table \ref{tab:mae-mse} and Table \ref{tab:crps-result}, respectively. The hidden dimension $H_d$ is selected from the set $\{64, 128, 256, 512\}$. Hyperparameters $K_1$ and $K_2$ are chosen from $\{0, 1, 2, \dots, \lfloor T/2 \rfloor + 1\}$. The kernel size for the moving average operation in DEMA is fixed at $a=25$. 
For reference, we provide a detailed hyperparameter configuration for FALDA with iTransformer as the backbone architecture in Table \ref{tab:config_1}. 
Furthermore, as discussed in Section \ref{sec:main_result}, we extend our framework to integrate with alternative backbone models (Autoformer, Transformer, and Informer), with their corresponding configurations detailed in Table \ref{tab:config_2}. All relevant hyperparameters referenced in Section \ref{sec:falda} are explicitly documented in these configuration tables.

\begin{table}[htbp]
\centering
\begin{minipage}{0.5\textwidth}  
\centering
\caption{Hyperparameter settings for FALDA with iTransformer backbone.}
\footnotesize 
\setlength{\tabcolsep}{2pt} 
\renewcommand{\arraystretch}{0.9} 
\newcommand{\numfont}{\fontsize{8.5pt}{9pt}\selectfont}

\begin{tabular}{@{}l*{6}{c}@{}} 
\toprule
& {\fontsize{8}{10}\selectfont Exchange} & {\fontsize{8}{10}\selectfont ILI} & {\fontsize{8}{10}\selectfont ETTm2} & {\fontsize{8}{10}\selectfont Electricity} & {\fontsize{8}{10}\selectfont Traffic} & {\fontsize{8}{10}\selectfont Weather} \\
\midrule
$\eta$       & \numfont 1.0      & \numfont 0.5 & \numfont 1.0   & \numfont 1.0         & \numfont 1.0     & \numfont 1.0     \\
$\delta$  & \numfont 0        & \numfont 0   & \numfont 1     & \numfont 2           & \numfont 1       & \numfont 0       \\
$\Delta$    & \numfont 3        & \numfont 3   & \numfont 10    & \numfont 10          & \numfont 20      & \numfont 3       \\
$K_1$     & \numfont 0        & \numfont 0   & \numfont 0     & \numfont 0           & \numfont 0       & \numfont 2       \\
$K_2$     & \numfont 4        & \numfont 2   & \numfont 5     & \numfont 20          & \numfont 3       & \numfont 25      \\
\bottomrule
\end{tabular}
\label{tab:config_1}
\end{minipage}
\hfill  
\begin{minipage}{0.5\textwidth}  
\centering
\caption{Hyperparameter settings for FALDA with other backbones.}
\footnotesize 
\setlength{\tabcolsep}{2pt} 
\renewcommand{\arraystretch}{0.9} 
\newcommand{\numfont}{\fontsize{8.5pt}{9pt}\selectfont}

\begin{tabular}{@{}l*{6}{c}@{}} 
\toprule
& {\fontsize{8}{10}\selectfont Exchange} & {\fontsize{8}{10}\selectfont ILI} & {\fontsize{8}{10}\selectfont ETTm2} & {\fontsize{8}{10}\selectfont Electricity} & {\fontsize{8}{10}\selectfont Traffic} & {\fontsize{8}{10}\selectfont Weather} \\
\midrule
$\eta$       & \numfont 1.0      & \numfont 0.5 & \numfont 1.0   & \numfont 1.0         & \numfont 1.0     & \numfont 1.0     \\
$\delta$  & \numfont 0        & \numfont 0   & \numfont 1     & \numfont 2           & \numfont 1       & \numfont 0       \\
$\Delta$    & \numfont 3        & \numfont 3   & \numfont 10    & \numfont 10          & \numfont 20      & \numfont 3       \\
$K_1$     & \numfont 2        & \numfont 2   & \numfont 5     & \numfont 0           & \numfont 30       & \numfont 2       \\
$K_2$     & \numfont 4        & \numfont 2   & \numfont 5     & \numfont 10          & \numfont 2       & \numfont 25      \\
\bottomrule
\end{tabular}
\label{tab:config_2}
\end{minipage}
\end{table}

\section{Additional Experimental Results}
\subsection{Ablation Study on Denoiser Architecture}\label{appendix:ab_denoiser}
As described in Section \ref{sec:FALDA-main-framework}, we introduce DEMA (Denoising MLP with Adaptive Layer Normalization), an MLP-based denoising module that utilizes Adaptive Layer Normalization (AdaLN) for feature transformation. The encoder layer employs a Moving Average (MA) operation to separate the latent variable into two components: seasonal and trend features. These components are then processed through independent AdaLN transformations, each governed by three trainable parameters: scale, shift, and gating coefficients, as specified in Equation \eqref{eq:denoiser:layernorm}. To evaluate the architectural decisions in DEMA, we compare against two baseline variants in Table \ref{tab:ablation:denoiser}:

\begin{itemize}
    \item \textbf{AD-MA}: This baseline removes the Moving Average decomposition in Eq. \eqref{eq:denoiser_ma}, applying AdaLN only to the undivided latent variable. While this configuration helps assess the importance of MA decomposition, it reduces the parameter count compared to DEMA. To address this confounding factor, we introduce a second controlled variant.
    
    \item \textbf{AD+LV}: This baseline maintains DEMA's parameter count while removing the feature decomposition step. Specifically, it implements two parallel AdaLN operations on the original latent variable (rather than on decomposed features). This design enables direct comparison of architectural contributions by isolating the effect of feature decomposition from pure parameter increases.
\end{itemize}

Experimental results demonstrate that DEMA consistently outperforms both variants in most datasets.

\begin{table}[htbp]
\centering
\scriptsize
\caption{Ablation study on denoiser architecture: comparison of DEMA and its variants. All experiments are repeated 10 times to compute the Means and Standard Deviation.}
\begin{tabular}{l|c|c|c|c|c|c}
\toprule
Dataset & \multicolumn{2}{c|}{DEMA} & \multicolumn{2}{c|}{AD-MA} & \multicolumn{2}{c}{AD+LV} \\
\cmidrule(lr){2-3} \cmidrule(lr){4-5} \cmidrule(lr){6-7}
 & MSE & MAE & MSE & MAE & MSE & MAE \\
\midrule
\multirow{1}{*}{Exchange} 
 & \textbf{0.180 $\pm$ 0.011} & \textbf{0.308 $\pm$ 0.009} & 0.197 $\pm$ 0.018 & 0.319 $\pm$ 0.014 & 0.183 $\pm$ 0.014 & 0.311 $\pm$ 0.010 \\

\multirow{1}{*}{ILI}
 & \textbf{1.652 $\pm$ 0.062} & {0.793 $\pm$ 0.026} & 1.735 $\pm$ 0.156 & 0.810 $\pm$ 0.058 & 1.666 $\pm$ 0.091 & \textbf{0.783 $\pm$ 0.031} \\

\multirow{1}{*}{ETTm2}
 & \textbf{0.250 $\pm$ 0.003} & \textbf{0.307 $\pm$ 0.003} & 0.250 $\pm$ 0.005 & 0.307 $\pm$ 0.004 & 0.252 $\pm$ 0.004 & 0.308 $\pm$ 0.002 \\

\multirow{1}{*}{Weather}
 & \textbf{0.217 $\pm$ 0.003} & \textbf{0.261 $\pm$ 0.004} & 0.220 $\pm$ 0.002 & 0.264 $\pm$ 0.004 & 0.219 $\pm$ 0.005 & 0.262 $\pm$ 0.005 \\
\bottomrule
\end{tabular}
\label{tab:ablation:denoiser}
\end{table}

\subsection{Does Diffusion Help? Frequency Decomposition Ablation Study}\label{appendix:ab_diffusion}
As analyzed in Appendix \ref{appendix: our_situation}, we introduce a temporal decomposition operation to strengthen the point forecasting capability of the backbone model, while the diffusion process primarily handles aleatoric uncertainty learning. To investigate whether probabilistic learning provides additional benefits to point forecasting, we conduct a comparative study with two deterministic models that exclude the diffusion component:

\begin{itemize}
    \item \textbf{NDB (Non-decomposed Backbone)}: The baseline backbone model without temporal decomposition operation.
    \item \textbf{DB (Decomposed Backbone)}: An enhanced architecture that incorporates (1) input decomposition that separates low-frequency noise components, and (2) an NS-adapter module for non-stationary feature learning.
\end{itemize}

As shown in Table \ref{tab:ab_backbone_diffusion}, the complete FALDA framework demonstrates superior performance compared to both deterministic variants (NDB and DB). These results suggest that: this decomposition operation effectively improves forecasting accuracy. Additionally, the diffusion component in FALDA provides additional performance gains beyond what can be achieved through decomposition alone. This empirical evidence confirms that probabilistic learning through diffusion modeling contributes positively to point forecasting performance when combined with our proposed decomposition architecture.

\begin{table}[h]
\centering
\caption{Ablation study on the benefits of probabilistic residual learning in forecasting performance.}
\label{tab:ab_backbone_diffusion}
\vspace{0.3cm}
\small
\setlength{\tabcolsep}{7pt}
\begin{tabular}{@{}l|cc|cc|cc|cc@{}}
\toprule
\multirow{2}{*}{Method} & \multicolumn{2}{c|}{Exchange} & \multicolumn{2}{c|}{ILI} & \multicolumn{2}{c|}{Electricity} & \multicolumn{2}{c@{}}{Traffic} \\
\cmidrule(lr){2-3}\cmidrule(lr){4-5}\cmidrule(lr){6-7}\cmidrule(l){8-9}
 & MAE & MSE & MAE & MSE & MAE & MSE & MAE & MSE \\
\midrule
Ours & \textbf{0.165} & \textbf{0.296} & \textbf{1.666} & \textbf{0.821} & \textbf{0.163} & \textbf{0.248} & \textbf{0.412} & \textbf{0.251} \\
NDB  & 0.194 & 0.315 & 1.786 & 0.826 & 0.165 & 0.249 & 0.439 & 0.276 \\
DB  & 0.194 & 0.316 & 1.791 & 0.828 & 0.165 & 0.250 & 0.439 & 0.276 \\
\bottomrule
\end{tabular}

\vspace{0.3cm}
\begin{minipage}{\linewidth}
\raggedright
\footnotesize

\end{minipage}
\end{table}

\subsection{Residual Framework Comparison with Identical Backbone} \label{appendix:exp:same backbone}
In this section, we evaluate the performance of TMDM, D3U, and FALDA with the NSformer backbone. The parameter configuration follows \cite{2024TMDM}, while the correlation results are reported in accordance with \cite{2025D3U}. Our experimental setup maintains consistency between the training and evaluation phases.
Table~\ref{tab:backbone_point} presents the point forecasting performance, measured by MAE and MSE. Meanwhile, Table~\ref{tab:backbone_probabilistic} summarizes the probabilistic forecasting performance using CRPS and CRPS$_{\text{sum}}$ metrics.
The experimental results demonstrate that FALDA achieves superior performance in both point and probabilistic forecasting tasks, validating the effectiveness of our proposed framework. By incorporating a time series decomposition mechanism to decouple distinct temporal components, our method facilitates more balanced learning of both epistemic and aleatoric uncertainties, thereby contributing to enhanced forecasting performance.
\begin{table}[h]
\centering
\footnotesize
\setlength{\tabcolsep}{3.5pt} 
\caption{Point forecasting performance comparison of different residual learning frameworks with NSformer backbone.}
\begin{tabular}{c|cc|cc|cc|cc|cc|cc}
\toprule
\multirow{2}{*}{Method} & \multicolumn{2}{c}{Exchange} & \multicolumn{2}{c}{ILI} & \multicolumn{2}{c}{ETTm2} & \multicolumn{2}{c}{Electricity} & \multicolumn{2}{c}{Traffic} & \multicolumn{2}{c}{Weather} \\
\cmidrule(lr){2-3}\cmidrule(lr){4-5}\cmidrule(lr){6-7}\cmidrule(lr){8-9}\cmidrule(lr){10-11}\cmidrule(lr){12-13}
 & MSE & MAE & MSE & MAE & MSE & MAE & MSE & MAE & MSE & MAE & MSE & MAE \\
\midrule
TMDM & 0.260 & 0.365 & 1.985 & 0.846 & 0.524 & 0.493 & 0.222 & 0.329 & 0.721 & 0.411 & 0.244 & 0.286 \\
D3U & 0.268 & 0.378 & 2.220 & 0.920 & \textbf{0.317} & 0.399 & 0.216 & 0.328 & 0.678 & 0.402 & \textbf{0.215} & \textbf{0.267} \\
Ours & \textbf{0.238} & \textbf{0.342} & \textbf{1.918} & \textbf{0.803} & {0.324} & \textbf{0.356} & \textbf{0.180} & \textbf{0.278} & \textbf{0.625} & \textbf{0.317} & 0.244 & 0.278 \\
\bottomrule
\end{tabular}
\label{tab:backbone_point}
\end{table}

\begin{table}[h]
\centering
\footnotesize
\setlength{\tabcolsep}{2.5pt} 
\begin{scriptsize} 
\caption{Probabilistic forecasting performance comparison of different residual learning frameworks with NSformer Backbone.}
\begin{tabular}{c|cc|cc|cc|cc|cc|cc}
\toprule
\multirow{2}{*}{Method} & \multicolumn{2}{c}{Exchange} & \multicolumn{2}{c}{ILI} & \multicolumn{2}{c}{ETTm2} & \multicolumn{2}{c}{Electricity} & \multicolumn{2}{c}{Traffic} & \multicolumn{2}{c}{Weather} \\
\cmidrule(lr){2-3}\cmidrule(lr){4-5}\cmidrule(lr){6-7}\cmidrule(lr){8-9}\cmidrule(lr){10-11}\cmidrule(lr){12-13}
 & CRPS & CRPS$_{\text{sum}}$  & CRPS & CRPS$_{\text{sum}}$ & CRPS & CRPS$_{\text{sum}}$ & CRPS & CRPS$_{\text{sum}}$ & CRPS & CRPS$_{\text{sum}}$ & CRPS & CRPS$_{\text{sum}}$ \\
\midrule
TMDM & 0.316 & 0.209 & 0.921 & 0.524 & 0.380 & 0.226 & 0.446 & \textbf{0.137} & 0.552 & \textbf{0.179} & 0.226 & 0.292 \\
D3U & 0.387 & 0.218 & 1.014 & 0.454 & \textbf{0.302} & \textbf{0.147} & 0.381 & 0.157 & 0.472 & 0.207 & \textbf{0.196} & \textbf{0.273} \\
Ours & \textbf{0.299} & \textbf{0.171} & \textbf{0.674} & \textbf{0.349} & 0.334 & 0.195 & \textbf{0.269} & {0.167} & \textbf{0.312} & {0.195} & 0.235 & 0.333 \\
\bottomrule
\end{tabular}
\label{tab:backbone_probabilistic}
\raggedright
\end{scriptsize}
\end{table}

\subsection{Training Strategy Experiments} \label{appendix:ab_training_strategy}
As defined in Eq.~\eqref{eq:loss_diffusion}, our loss function incorporates both a diffusion loss for denoiser optimization and a fine-tuning loss $\mathcal{L}_{\text{finetune}} = \|R - \text{sg}(\hat{R}^{(0)}_{\theta}(R^{(k')}, k', c))\|^2$ to simultaneously enhance the point estimate models. The hyperparameter $k'$ allows for flexible selection of diffusion steps during fine-tuning. To validate this choice, we perform ablation studies comparing models trained with and without fine-tuning, as well as models fine-tuned at different diffusion steps $k'$.
The experimental results presented in Figure~\ref{fig:ab_finetune} demonstrate that the fine-tuning operation provides consistent improvements over the no-fine-tuning setting. Additionally, our chosen configuration with $k'=100$ achieves competitive MSE and MAE performance among different step selections, suggesting the validity of our configuration as mentioned in Appendix~\ref{appendix:implementation}.
\begin{figure}[hbtp]
    \centering
    \includegraphics[width=1.0 \linewidth]{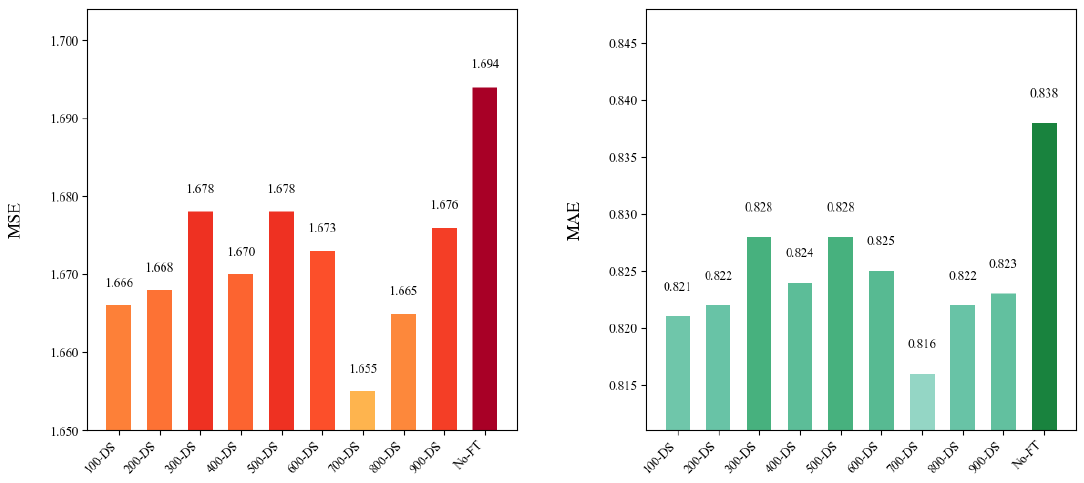}
    \caption{Evaluation of different training strategies on the ILI Dataset. The left subplot shows the MSE performance, while the right subplot shows the MAE performance. $k'$-DS: fine-tuning with diffusion step $k'$. No-FT: no fine-tuning.} 
    \label{fig:ab_finetune}
\end{figure}

\subsection{Training and Inference Efficiency} 
\label{appendix:train_infer_efficiency}
As discussed in Section \ref{sec:falda}, FALDA reconstructs the sample directly, rather than learning the noise at each diffusion step during the training phase, which reduces the learning complexity of the time series component. Additionally, our denoiser DEMA, which is designed as a lightweight MLP architecture, alleviates the training burden. During the inference process, we employ DDIM to accelerate inference. These design choices collectively contribute to the efficiency of FALDA, while maintaining its effectiveness. We conduct experiments to demonstrate its efficiency. As depicted in Figure~\ref{fig:speed_compare_exchange}, FALDA exhibits superior convergence properties compared to TMDM.
While TMDM requires approximately 30 epochs to converge on the Exchange dataset, FALDA achieves competitive performance after only 1 epoch. For fair comparison, we maintain identical training configurations with TMDM, including the learning rate ($1\times10^{-4}$) and optimization method (Adam optimizer). This accelerated convergence further underscores FALDA's computational advantages without compromising model performance.

\begin{figure}[hbtp]
    \centering
    \includegraphics[width=0.75\linewidth]{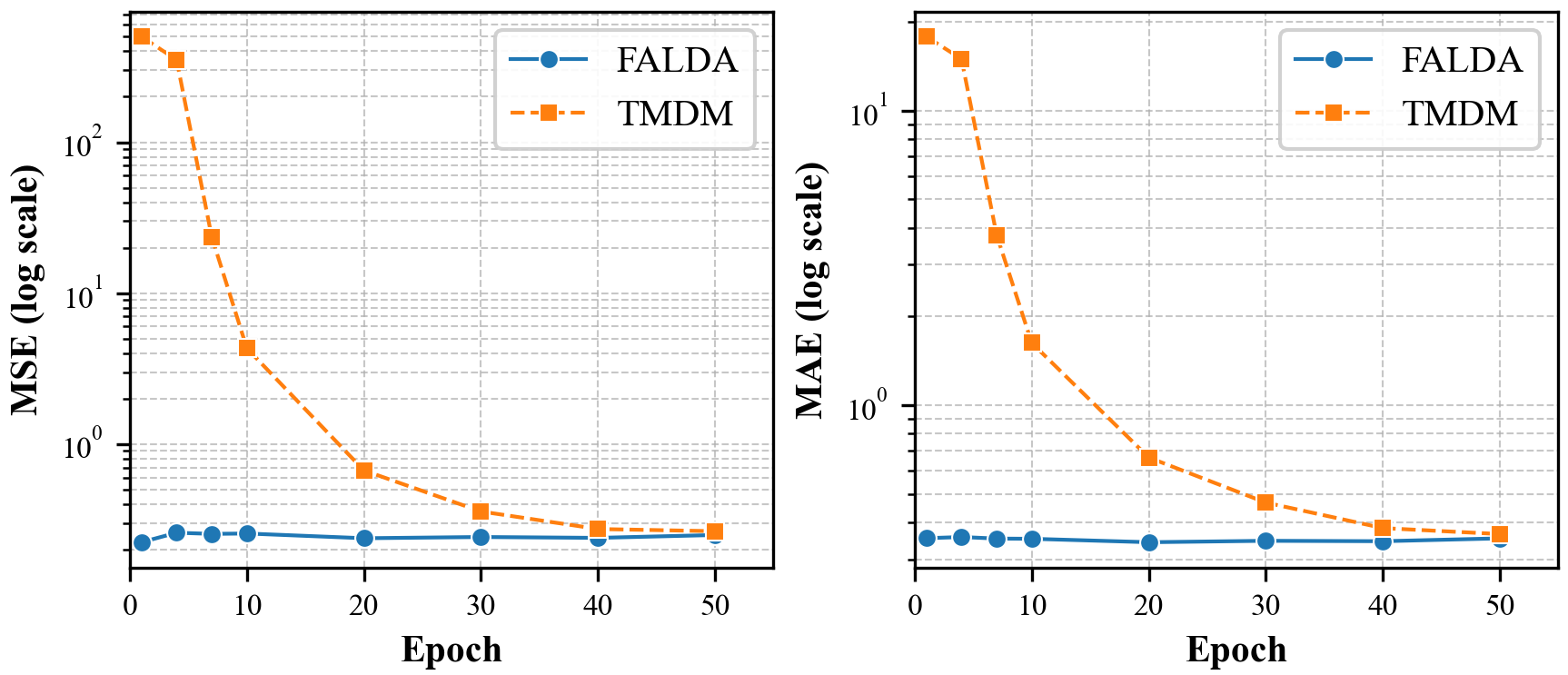}
    \caption{Training speed comparison between FALDA and TMDM on the Exchange dataset. The curves depict the evolution of metrics: MSE (left) and MAE (right) across training epochs.} 
    \label{fig:speed_compare_exchange}
\end{figure}

Building upon these convergence advantages, we implement a reduced early stopping patience for FALDA compared to TMDM during the training process, as detailed in Appendix~\ref{appendix:implementation}. During inference, we employ DDIM (Denoising Diffusion Implicit Models) to accelerate the reverse diffusion process, thereby significantly reducing both inference time and memory requirements.
Table~\ref{tab:appendix:train_test_speed} presents a comprehensive computational efficiency comparison between TMDM and FALDA across six benchmark datasets. The results demonstrate FALDA's consistent superiority in both training and inference phases. Specifically, FALDA achieves an inference speed improvement of up to 26.3$\times$ on the ETTm2 dataset, while attaining a training speed enhancement of up to 13.7$\times$ on the Exchange dataset. Furthermore, FALDA delivers a 2.1$\times$ training speed-up on the Electricity dataset (from 122.9 minutes to 58.3 minutes) and a 2.9$\times$ inference speed-up on the Traffic dataset (from 472.3 minutes to 160.7 minutes). These substantial improvements in computational efficiency not only validate FALDA's practical utility for real-world applications but also highlight its capability for processing high-dimensional datasets.

\begin{table}[h]
\centering
\caption{Comparison of training and inference times (minutes) between TMDM and FALDA \textsuperscript{1}.}
\label{tab:appendix:train_test_speed}
\vspace{0 cm}

\setlength{\tabcolsep}{8pt}
\begin{tabular}{@{}l|c|c|c|c@{}}
\toprule
\multirow{2}{*}{Dataset} & \multicolumn{2}{c}{\bfseries TMDM} & \multicolumn{2}{c}{\bfseries FALDA (Ours)} \\
\cmidrule(lr){2-3}\cmidrule(l){4-5}
 & \multicolumn{1}{c}{Training} & \multicolumn{1}{c|}{Inference \textsuperscript{2}} & \multicolumn{1}{c}{Training} & \multicolumn{1}{c}{Inference} \\
\midrule
ILI & 3.0 & 0.6 & \textbf{0.4} & \textbf{0.1} \\
Exchange Rate & 9.6 & 10.5 & \textbf{0.7} & \textbf{0.5} \\
ETTm2 & 36.6 & 194.4 & \textbf{3.4} & \textbf{7.4} \\
\addlinespace[2pt]
Weather & 69.8 & 119.1 & \textbf{6.3} & \textbf{13.5} \\
Electricity & 122.9 & 272.9 & \textbf{58.3} & \textbf{88.3} \\
Traffic & 97.0 & 472.3 & \textbf{83.6} & \textbf{160.7} \\
\bottomrule
\end{tabular}

\vspace{0 cm}
\begin{minipage}{0.9\linewidth}
\raggedright
\footnotesize
\textsuperscript{1} All experiments were conducted on an NVIDIA L20 GPU with 48GB memory.

\textsuperscript{2} Inference times were measured with 100 samples per test instance.
\end{minipage}
\end{table}

\subsection{Predictive Intervals result}
We present the result of PICE and QICE in Tabel \ref{tab:pice_qice}, which assesses the ability of the model to accurately cover the true values within its prediction intervals and the precision of the estimates for these intervals, respectively. See Appendix \ref{appendix:metrics} for specific definitions of PICE and QICE.

\begin{table}[htbp]
\centering
\caption{Comparison of PICP and QICE metrics.}
\label{tab:pice_qice}
\resizebox{0.8 \linewidth}{!}{ 
\begin{tabular}{lcccccc}
\toprule
\textbf{Dataset} & \textbf{Metric} & \textbf{TimeGrad} & \textbf{CSDI} & \textbf{TimeDiff} & \textbf{TMDM} & \textbf{Ours} \\
\midrule

\multirow{2}{*}{Exchange} 
 & PICP & 69.16 & 69.21 & 20.80 & 74.54 & \textbf{97.88 }\\
 & QICE & 5.32 & 5.49 & 13.34 & \textbf{4.38} & 5.49 \\
\midrule

\multirow{2}{*}{ILI} 
 & PICP & 74.29 & 76.18 & 3.69 & \textbf{87.83} & 77.49 \\
 & QICE & 7.86 & 7.75 & 15.50 & {6.74} & \textbf{4.42} \\
\midrule

\multirow{2}{*}{ETTm2} 
 & PICP & 71.62 & 71.78 & 13.16 & 73.20 & \textbf{84.08} \\
 & QICE & 5.37 & 5.07 & 14.22 & 3.75 & \textbf{2.71} \\
\midrule



\multirow{2}{*}{Weather} 
 & PICP & 62.79 & 62.71 & 21.60 & 72.97 & \textbf{80.75} \\
 & QICE & 7.36 & 5.14 & 13.18 & {3.87} & \textbf{3.58} \\
\bottomrule

\end{tabular}}
\end{table}

\section{Limitations} \label{appendix:limitations}
Although FALDA demonstrates competitive performance, it still has limitations. Specifically, the time series decomposition relies on removing the top $K_1$ and the last $K_2$ frequencies, where $K_1$ and $K_2$ are treated as hyperparameters. However, these hyperparameters play a critical role in influencing the overall framework performance, inadequate selection may lead to a decline in the backbone's performance. Future research could explore more systematic approaches to selecting $K_1$ and $K_2$, such as incorporating learnable mechanisms or other adaptive methods.

\section{Showcases}

\subsection{Case Study of FALDA and TMDM}
To demonstrate the superior probabilistic forecasting capability of FALDA, we present comparative visualizations of ground truth values and prediction results between FALDA and TMDM across four datasets in Figures~\ref{fig:total-ili}, \ref{fig:total-exchange}, \ref{fig:total-weather}, and~\ref{fig:total-ettm2}. The figures display the predicted median along with 50\% and 90\% distribution intervals, where the lower and upper percentiles are set at 2.5\% and 97.5\%, respectively.

Our experimental results demonstrate that FALDA achieves significantly better point forecasting accuracy compared to TMDM. Moreover, the residual learning approach combined ensures particularly accurate predictions for the first future time step, as clearly evidenced in Figure~\ref{fig:total-exchange}. While TMDM produces excessively wide prediction intervals for the initial future prediction, FALDA generates precise first-step forecasts with narrow confidence bounds that gradually widen over time. This behavior aligns well with real-world time series characteristics, where continuous variation is typically observed. Given complete historical information, especially the most recent observations, the immediate future time step should not deviate drastically from the last observed value. This forecasting model is particularly well suited to financial data, whose volatility typically increases over time, a phenomenon which corresponds well with our experimental results. The findings of this study indicate that FALDA not only provides more accurate forecasts, but also produces results that are more interpretable and better reflect the underlying data dynamics.

\subsection{Time Series Decomposition Visualization}
To illustrate our time series decomposition approach, Figures~\ref{fig:visualize_decomposition_1} and~\ref{fig:visualize_decomposition_2} visualize the distinct temporal components obtained through the decomposition method described in Equation~\eqref{eq:decomposition}. Figures~\ref{fig:visualize_decomposition_1} and~\ref{fig:visualize_decomposition_2} demonstrate the distinct decomposition characteristics of iTransformer and other backbones, respectively. The detailed implementation settings for these decomposition strategies are provided in Appendix~\ref{appendix:implementation}.

\subsection{Visualization of Key Components in FALDA}
To further showcase the predictive capabilities of FALDA, we visualize the outputs of its three key components across different datasets with the TS-Backbone set as Autoformer. Figures \ref{fig:vis_ili_multi-comp}, \ref{fig:vis_exchange_multi-comp}, \ref{fig:vis_ettm2_multi-comp}, and \ref{fig:vis_traffic_multi-comp} display both the model’s overall predictions and the decomposed predictions for the non-stationary term, the stationary term, and the noise term. For clarity, we sample 100 predictions to represent the probabilistic learning outcomes, with the width of the prediction intervals indicating the model’s quantified uncertainty.

Figure \ref{fig:vis_exchange_multi-comp} shows the progressive widening of prediction intervals over time, a pattern that aligns with the inherent characteristics typically observed in financial data. A comparative analysis of these figures reveals an inverse correlation between prediction accuracy and the width of the prediction intervals: more precise point estimates are associated with narrower uncertainty bounds, which is consistent with our residual learning paradigm.
These findings underscore FALDA’s capacity to effectively model aleatoric uncertainty across diverse datasets, while simultaneously preserving high predictive accuracy. This dual capability highlights the model's strength in both uncertainty quantification and forecasting precision.

\begin{figure}[htbp]
    \centering
    \begin{subfigure}[b]{1\textwidth}
        \includegraphics[width=\linewidth]{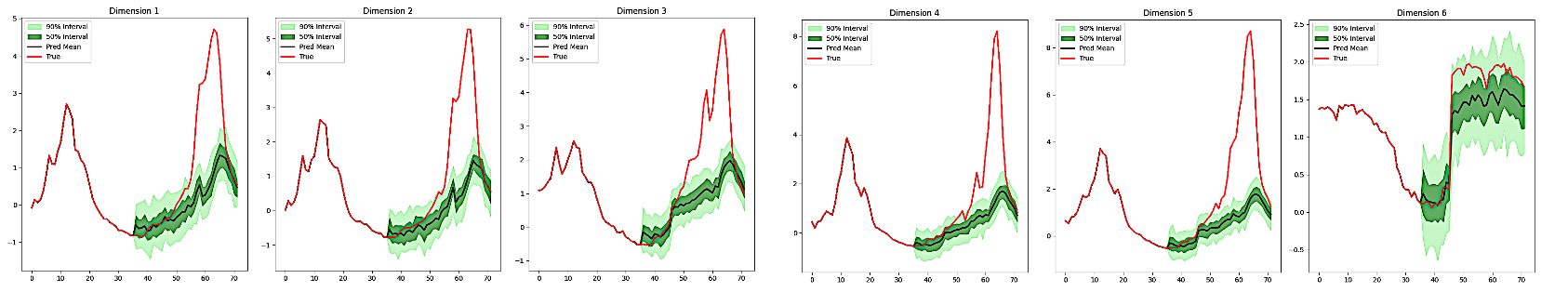} 
        \caption{TMDM}
        \label{fig:sub1-tmdm-ili}
    \end{subfigure}
    \vspace{0.05cm}
    
    \begin{subfigure}[b]{1\textwidth}
        \includegraphics[width=\linewidth]{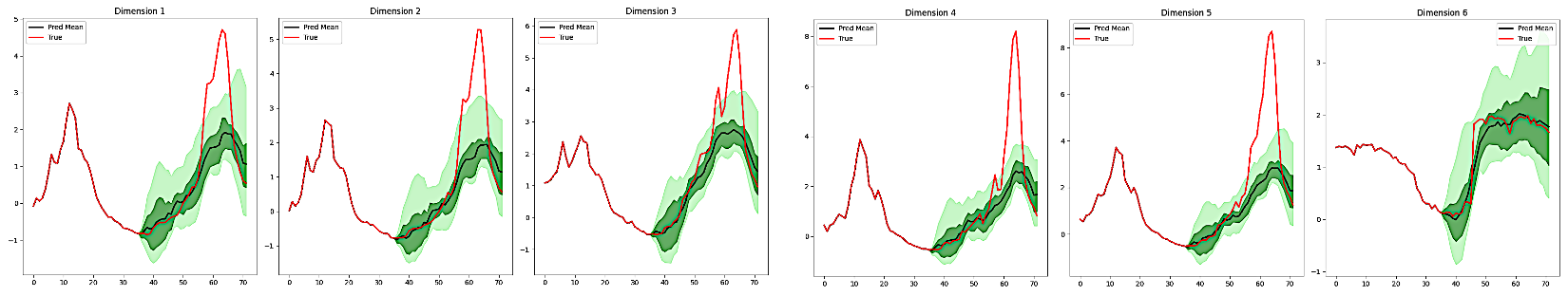} 
        \caption{FALDA}
        \label{fig:sub2-falda-ili}
    \end{subfigure}
    
    \caption{Comparison of prediction intervals for the ILI dataset ($T=36, S=36$). The red line indicates the ground truth, and the black line represents the predicted mean. Dark green shading denotes the 50\% prediction interval, and light green shading shows the 90\% prediction interval.}
    \label{fig:total-ili}
\end{figure}

\begin{figure}[htbp]
    \centering
    \begin{subfigure}[b]{1\textwidth}
        \includegraphics[width=\linewidth]{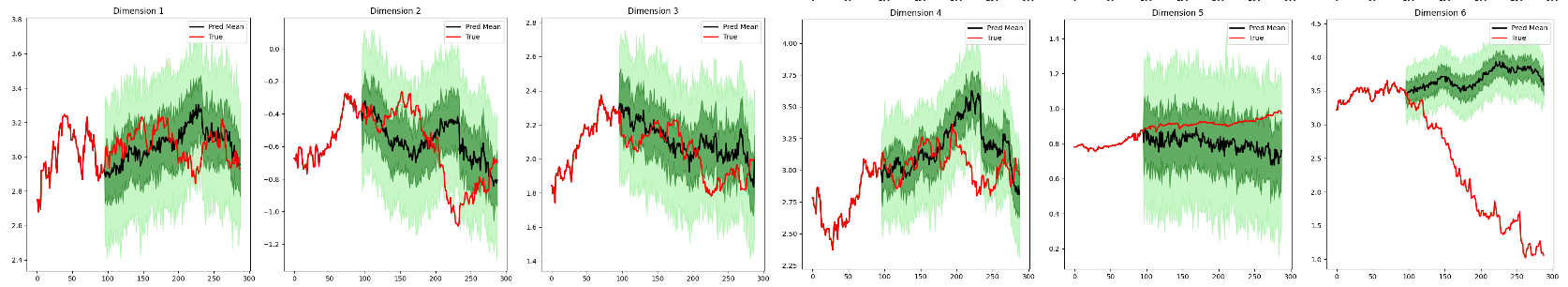} 
        \caption{TMDM}
        \label{fig:sub1-tmdm-exchange}
    \end{subfigure}
    
    \vspace{0.05cm}
    \begin{subfigure}[b]{1\textwidth}
        \includegraphics[width=\linewidth]{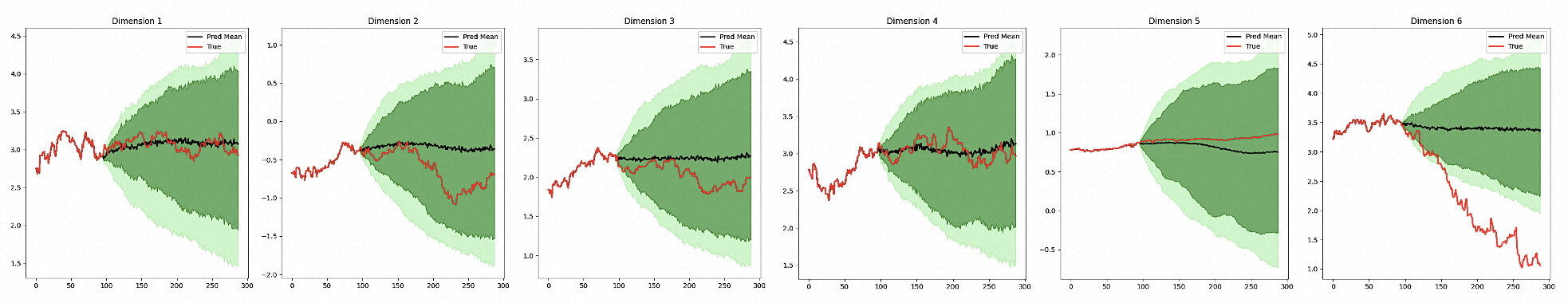} 
        \caption{FALDA}
        \label{fig:sub2-falda-exchange}
    \end{subfigure}
    \caption{Comparison of prediction intervals for the Exchange dataset ($T=96, S=192$). The red line indicates the ground truth, and the black line represents the predicted mean. Dark green shading denotes the 50\% prediction interval, and light green shading shows the 90\% prediction interval.}
    \label{fig:total-exchange}
\end{figure}

\begin{figure}[htbp]
    \centering
    \begin{subfigure}[b]{1\textwidth}
        \includegraphics[width=\linewidth]{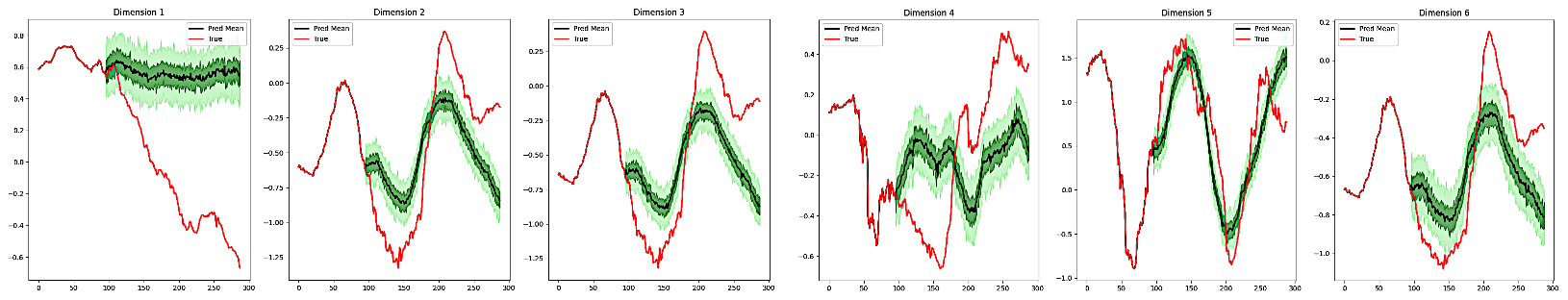} 
        \caption{TMDM}
        \label{fig:sub1-tmdm-weather}
    \end{subfigure}
    
    \vspace{0 cm}
    \begin{subfigure}[b]{1\textwidth}
        \includegraphics[width=\linewidth]{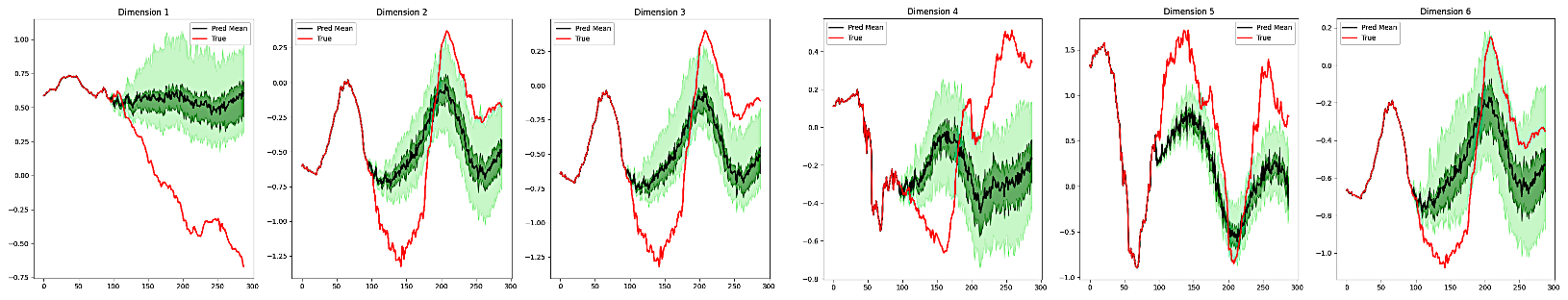} 
        \caption{FALDA}
        \label{fig:sub2-falda-weather}
    \end{subfigure}
    \caption{Comparison of prediction intervals for the Weather dataset ($T=96, S=192$). The red line indicates the ground truth, and the black line represents the predicted mean. Dark green shading denotes the 50\% prediction interval, and light green shading shows the 90\% prediction interval.}
    \label{fig:total-weather}
\end{figure}

\begin{figure}[htbp]
    \centering
    \begin{subfigure}[b]{1\textwidth}
        \includegraphics[width=\linewidth]{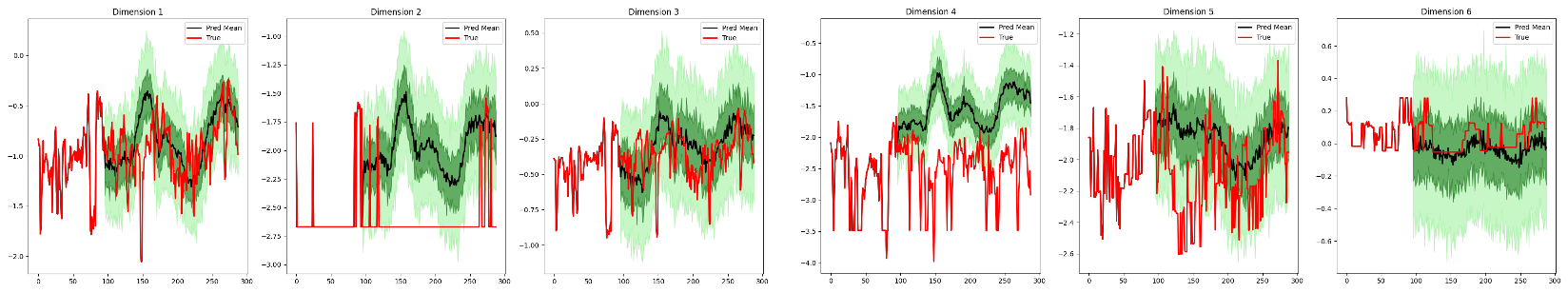} 
        \caption{TMDM}
        \label{fig:sub1-tmdm-ettm2}
    \end{subfigure}
    \vspace{0 cm}
    \begin{subfigure}[b]{1\textwidth}
        \includegraphics[width=\linewidth]{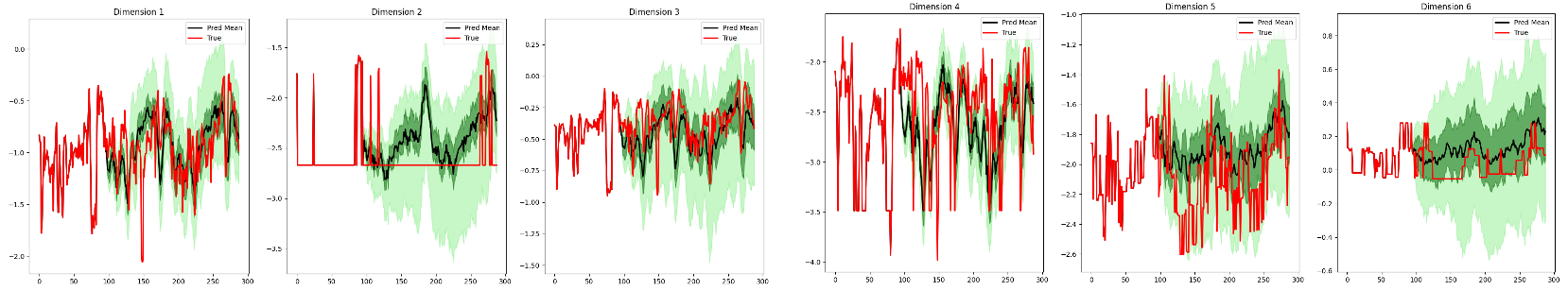} 
        \caption{FALDA}
        \label{fig:sub2-falda-ettm2}
    \end{subfigure}
    \caption{Comparison of prediction intervals for the ETTm2 dataset ($T=96, S=192$). The red line indicates the ground truth, and the black line represents the predicted mean. Dark green shading denotes the 50\% prediction interval, and light green shading shows the 90\% prediction interval.}
    \label{fig:total-ettm2}
\end{figure}

\begin{figure}
    \centering
    \includegraphics[width=1.0\linewidth]{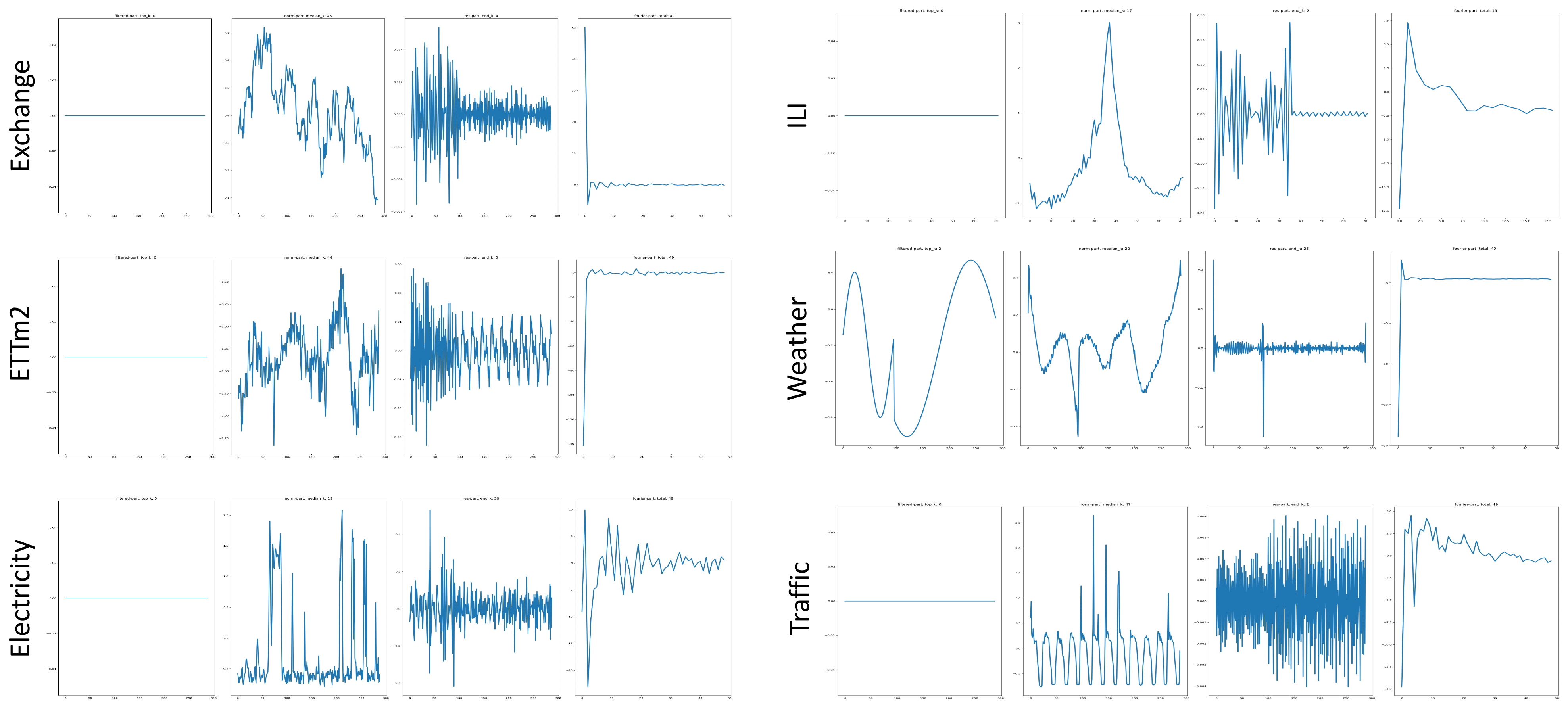}
    \caption{Time series decomposition strategy for the iTransformer backbone. From left to right, the subfigures present: (1) the non-stationary term, (2) the stationary term, (3) the noise term, and (4) the frequency-domain representation obtained via Fourier transform.}
    \label{fig:visualize_decomposition_1}
\end{figure}

\begin{figure}
    \centering
    \includegraphics[width=1\linewidth]{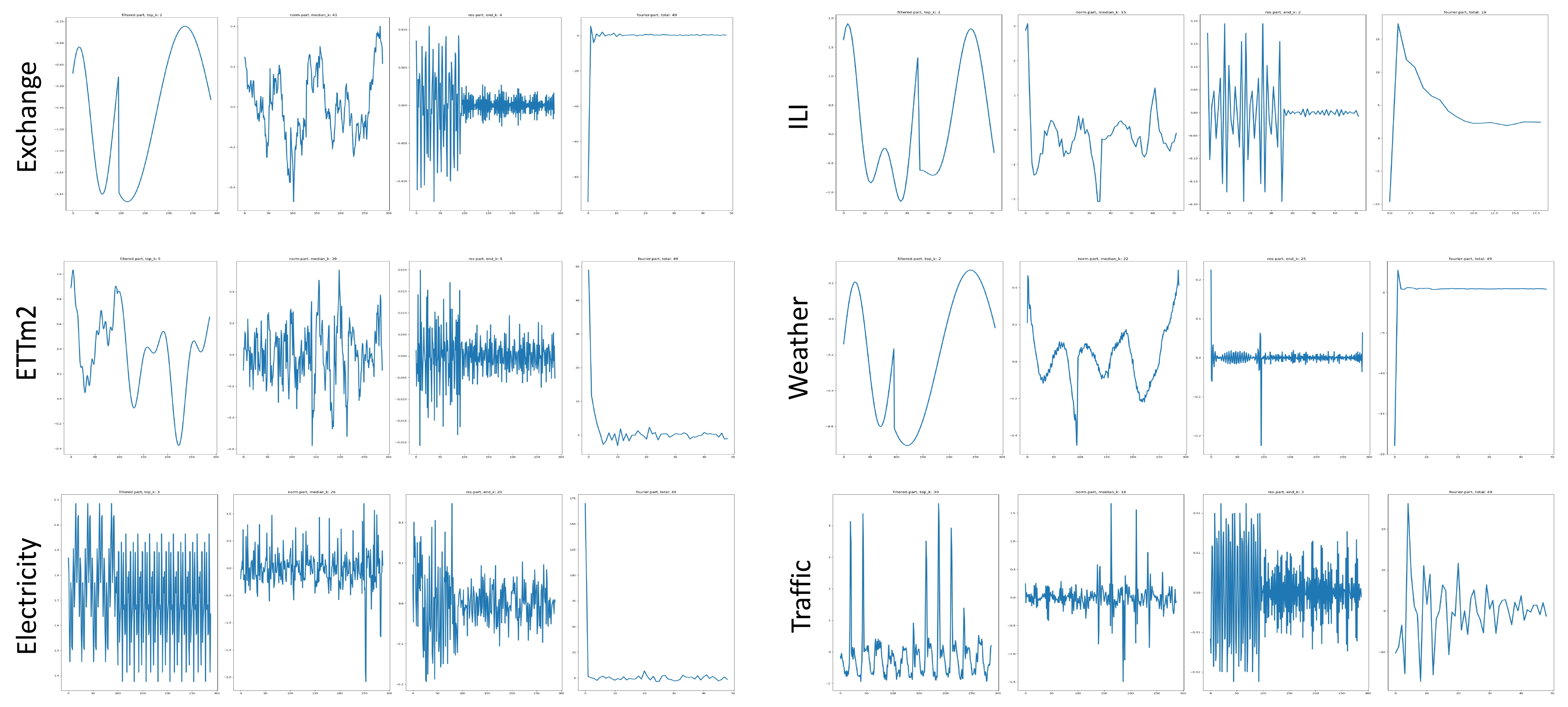}
    \caption{Time series decomposition strategy for other backbones. From left to right, the subfigures present: (1) the non-stationary term, (2) the stationary term, (3) the noise term, and (4) the frequency-domain representation obtained via the Fourier transform.}
    \label{fig:visualize_decomposition_2}
\end{figure}

\begin{figure}[htbp]
    \centering
    \begin{subfigure}[b]{0.3\textwidth}
        \centering
        \includegraphics[width=\textwidth]{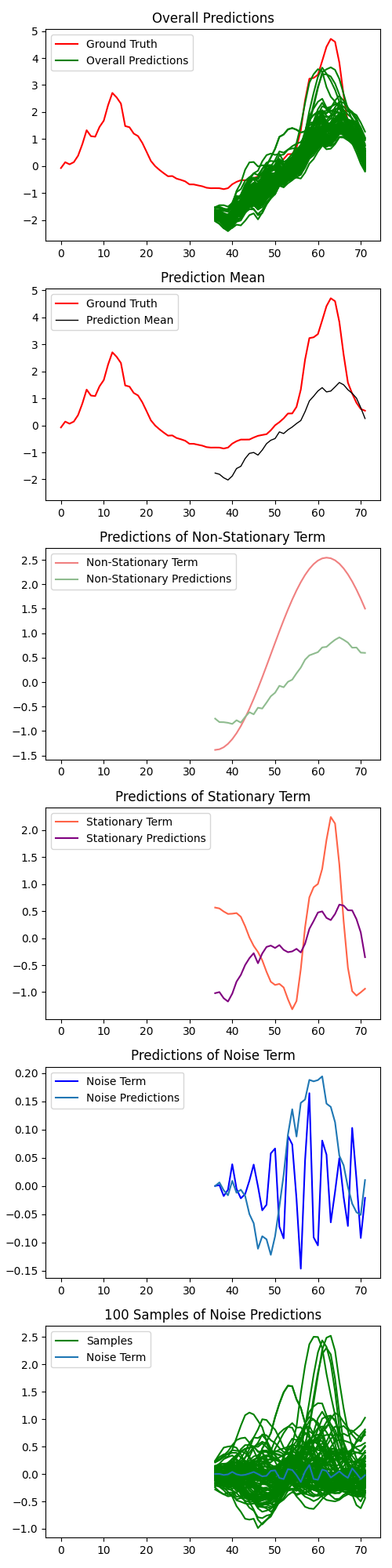} 
        \caption{ILI 1$^{\text{th}}$ dimension}
    \end{subfigure}
    \hspace{10pt}
    \begin{subfigure}[b]{0.3\textwidth}
        \centering
        \includegraphics[width=\textwidth]{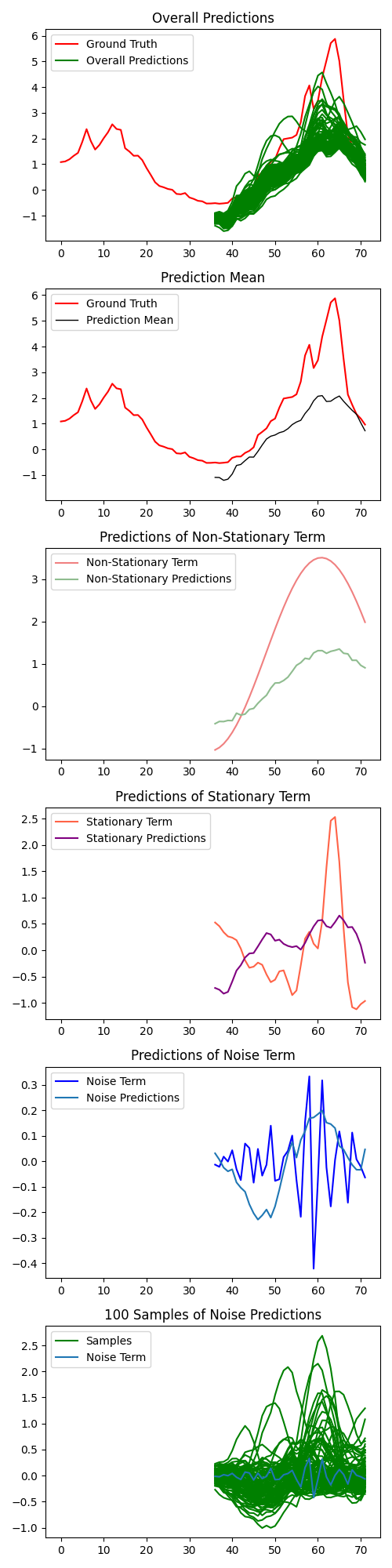} 
        \caption{ILI 3$^{\text{th}}$ dimension}
    \end{subfigure}
    \hspace{10pt}
    \begin{subfigure}[b]{0.3\textwidth}
        \centering
        \includegraphics[width=\textwidth]{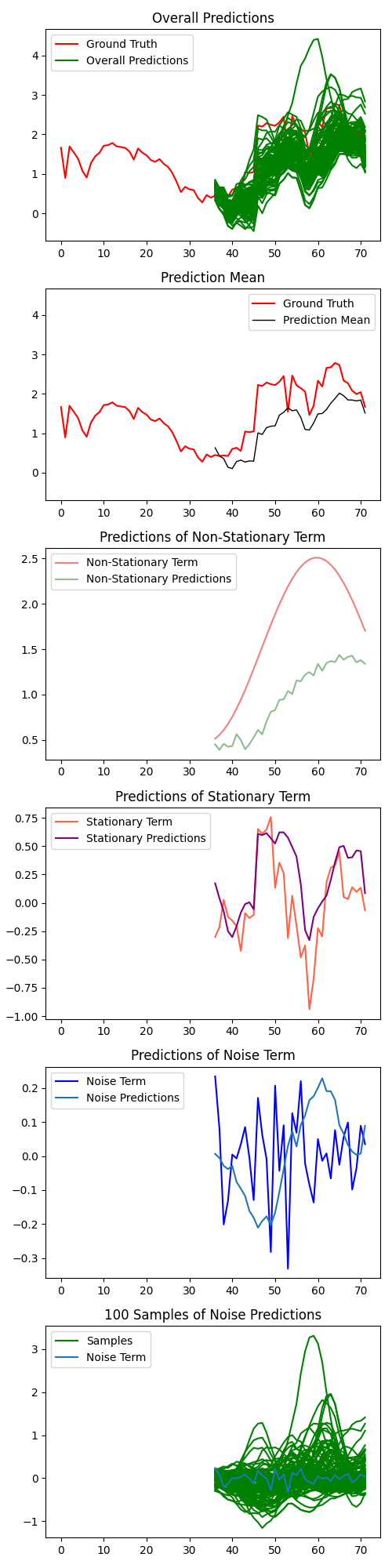} 
        \caption{ILI 7$^{\text{th}}$ dimension}
    \end{subfigure}
    \caption{Visualization of the prediction results from the different components (NS-Adapter, TS-Backbone, and DEMA) on the ILI dataset.}
    \label{fig:vis_ili_multi-comp}
\end{figure}

\begin{figure}
    \centering
    \includegraphics[width=0.9\linewidth]{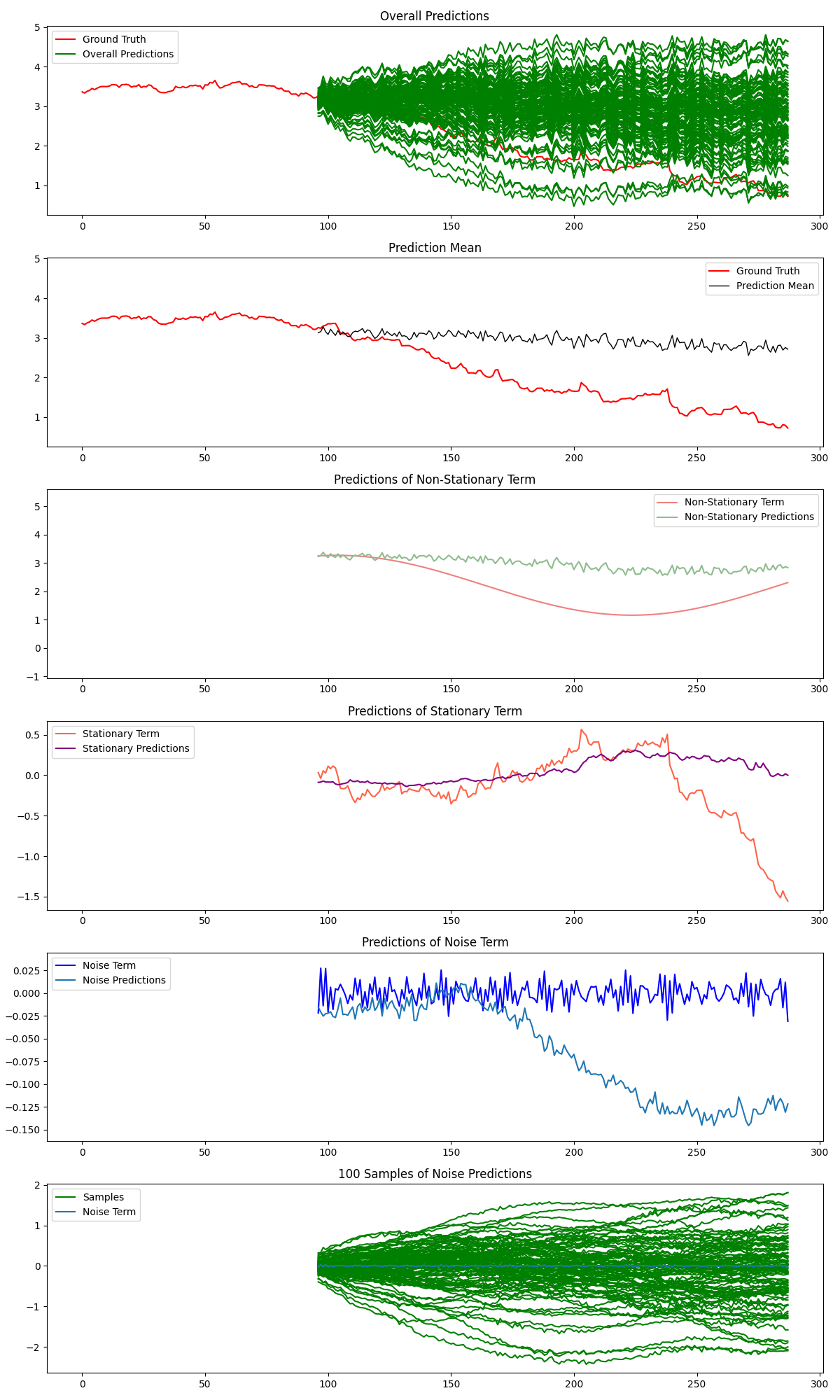}
    \caption{Visualization of the prediction results from the different components (NS-Adapter, TS-Backbone, and DEMA) on the Exchange dataset (6 $^{\text{th}}$ dimension).}
    \label{fig:vis_exchange_multi-comp}
\end{figure}

\begin{figure}
    \centering
    \includegraphics[width=0.9\linewidth]{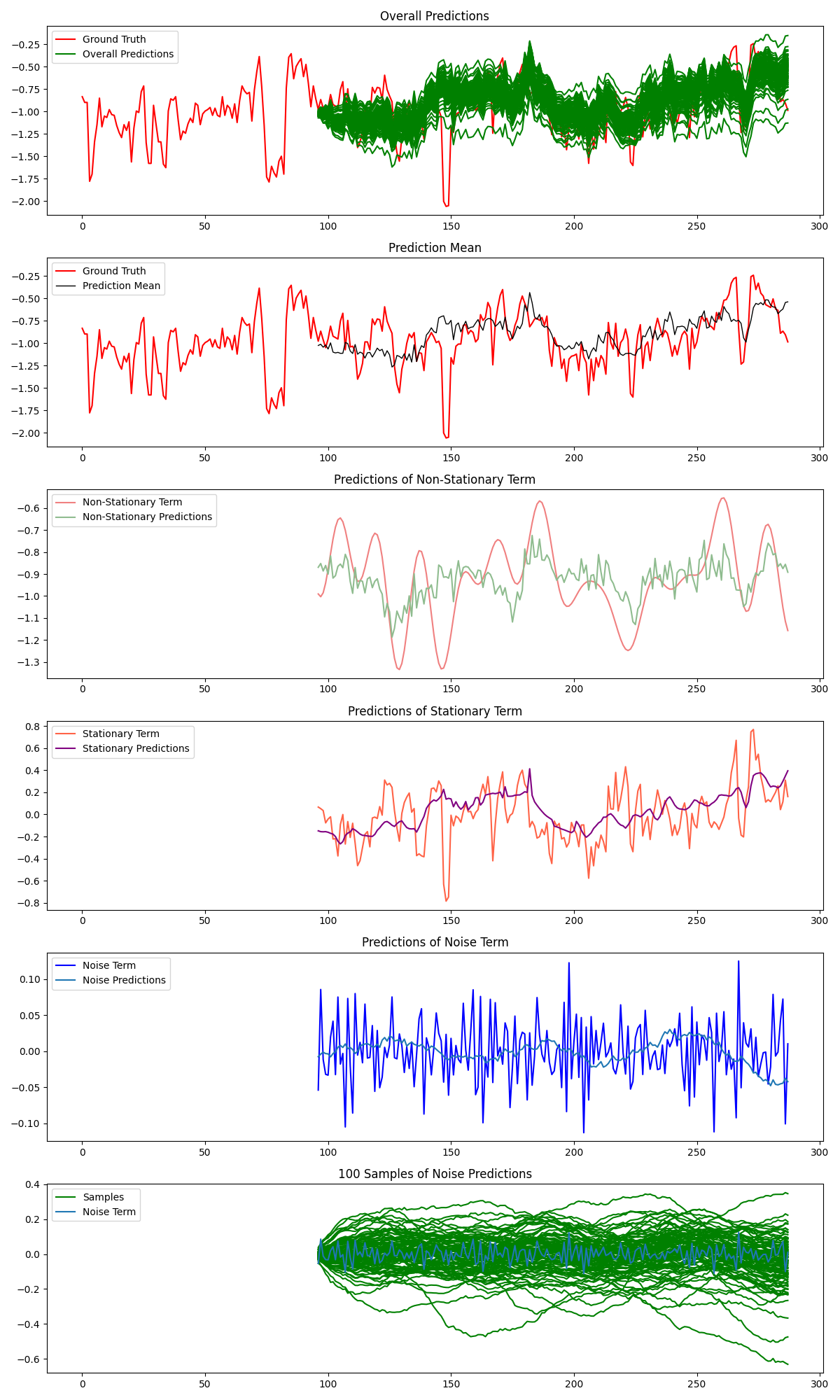}
    \caption{Visualization of the prediction results from the different components (NS-Adapter, TS-Backbone, and DEMA) on the ETTm2 dataset (1 $^{\text{th}}$ dimension).}
    \label{fig:vis_ettm2_multi-comp}
\end{figure}

\begin{figure}
    \centering
    \includegraphics[width=0.9\linewidth]{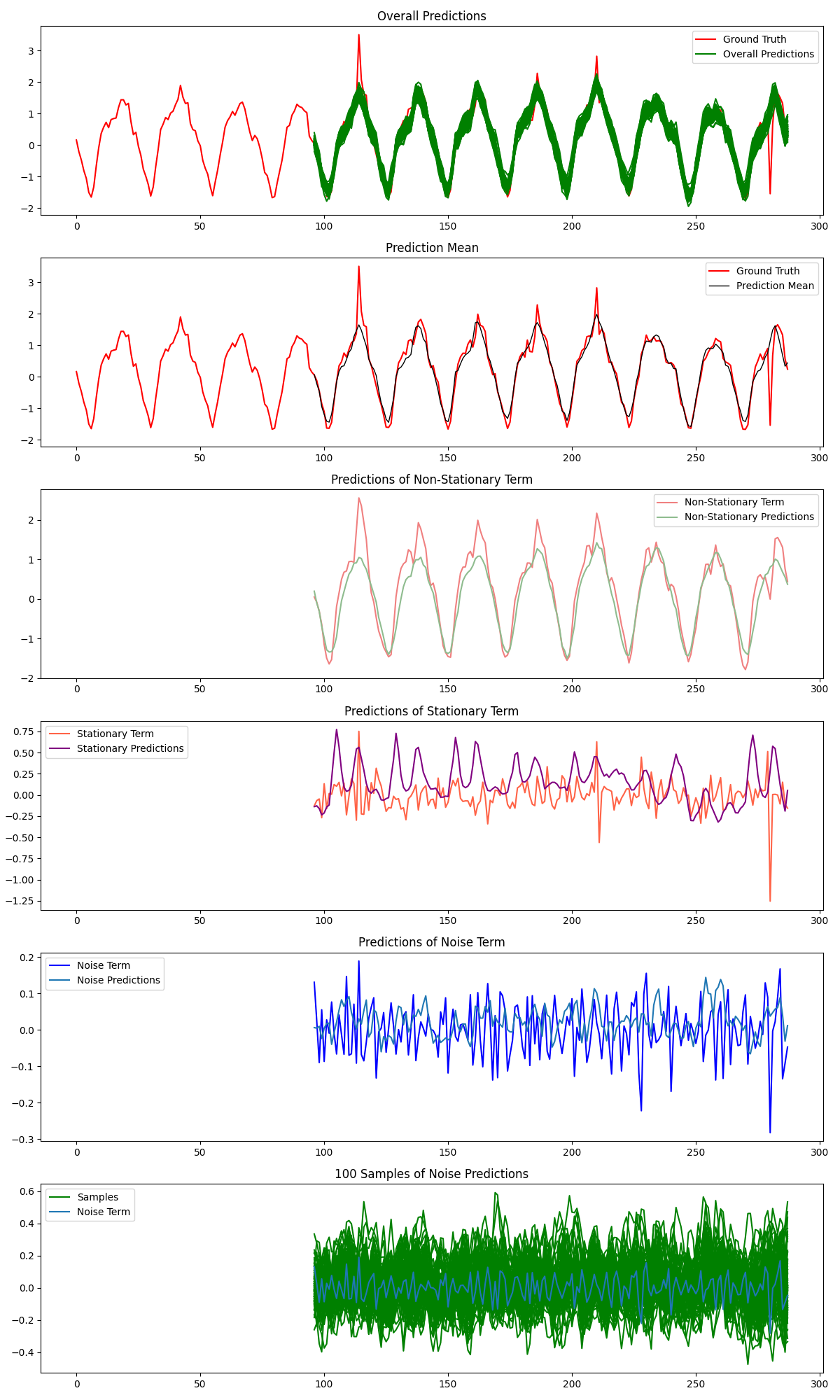}
    \caption{Visualization of the prediction results from the different components (NS-Adapter, TS-Backbone, and DEMA) on the Traffic dataset (800 $^{\text{th}}$ dimension).}
    \label{fig:vis_traffic_multi-comp}
\end{figure}

\clearpage

\end{document}